\documentclass{article}

\usepackage[preprint]{neurips_2019}

\usepackage[utf8]{inputenc} 
\usepackage[T1]{fontenc}    
\usepackage{hyperref}       
\usepackage{url}            
\usepackage{booktabs}       
\usepackage{amsfonts}       
\usepackage{nicefrac}       
\usepackage{microtype}      
\usepackage{ulem}

\usepackage{subcaption}
\usepackage{amsmath}
\usepackage{graphicx}
\usepackage{wrapfig}
\usepackage{multirow}
\usepackage{amsthm}
\usepackage{color}
\usepackage{floatrow}
\newfloatcommand{capbtabbox}{table}[][\FBwidth]
\newtheorem{proposition}{Proposition}
\newtheorem{definition}{Definition}
\newtheorem{lemma}{Lemma}
\newcommand{\hid}[2]{h^{(#1)}_{#2}}
\newcommand{\cid}[1]{c_{#1}}
\newcommand{\zid}[1]{z_{#1}}
\newcommand{\jid}[1]{J^{(#1)}}
\newcommand{\jhd}[1]{{\hat{J}}^{(#1)}}
\newcommand{\jidd}[2]{J^{(#1)}_{#2}}

\title{Theoretical and Experimental Analysis on the Generalizability of Distribution Regression Network}

\author{%
	Connie Kou\\
	School of Computing, National University of Singapore \\
	Bioinformatics Institute, A*STAR Singapore \\
	\texttt{koukl@comp.nus.edu.sg} \\
	\And
	Hwee Kuan Lee \\
	School of Computing, National University of Singapore \\
	Bioinformatics Institute, A*STAR Singapore \\
	\texttt{leehk@bii.a-star.edu.sg} \\
	\And
	Jorge Sanz \\
	School of Business, National University of Singapore \\
	School of Computing, National University of Singapore \\
	\texttt{jorges@nus.edu.sg} \\
	\And
	Teck Khim Ng \\
	School of Computing, National University of Singapore \\
	\texttt{ngtk@comp.nus.edu.sg} \\
}
\begin{document}

\maketitle

\begin{abstract}
There is emerging interest in performing regression between distributions. In contrast to prediction on single instances, these machine learning methods can be useful for population-based studies or on problems that are inherently statistical in nature. The recently proposed distribution regression network (DRN) \citep{KOU2018} has shown superior performance for the distribution-to-distribution regression task compared to conventional neural networks. However, in \citet{KOU2018} and some other works on distribution regression, there is a lack of comprehensive comparative study on both theoretical basis and generalization abilities of the methods. We derive some mathematical properties of DRN and qualitatively compare it to conventional neural networks. We also perform comprehensive experiments to study the generalizability of distribution regression models, by studying their robustness to limited training data, data sampling noise and task difficulty. DRN consistently outperforms conventional neural networks, requiring fewer training data and maintaining robust performance with noise. Furthermore, the theoretical properties of DRN can be used to provide some explanation on the ability of DRN to achieve better generalization performance than conventional neural networks.
\end{abstract}

\section{Introduction}
There has been emerging interest in perform regression on complex inputs such as probability distributions. Performing prediction on distributions has many important applications. Many real-world systems are driven by stochastic processes. For instance, the Fokker-Planck equation \citep{risken1996fokker} has been used to model a time-varying distribution, with applications such as astrophysics \citep{noble2011modeling}, biological physics \citep{guerin2011bidirectional} and weather forecasting \citep{palmer2000predicting}. Extrapolating a time-varying distribution has also been used to train a classifier where the data distribution drifts over time \citep{lampert2015predicting}.

A recently proposed distribution regression model, distribution regression network (DRN) \citep{KOU2018}, outperforms conventional neural networks by proposing a novel representation of encoding an entire distribution in a single network node. On the datasets used by \citet{KOU2018}, DRN achieves better accuracies with 500 times fewer parameters compared to the multilayer perceptron (MLP) and one-dimensional convolutional neural network (CNN). However, in \citet{KOU2018} and other distribution regression methods \citep{oliva2013distribution,oliva2015fast}, there is a lack of comprehensive comparative study on both theoretical basis and generalization abilities of the methods. 

In this work, we derive some mathematical properties of DRN and qualitatively compare it to conventional neural networks. We also performed comprehensive experiments to study the generalizability of distribution regression models, by studying their robustness to limited training data, data sampling noise and task difficulty. DRN consistently outperforms conventional neural networks, requiring two to five times fewer training data to achieve similar generalization performance. With increasing data sampling noise, DRN's performance remains robust whereas the neural network models saw more drastic decrease in test accuracy. Furthermore, the theoretical properties of DRN can be used to provide insights on the ability of DRN to achieve better generalization performance than conventional neural networks.

\section{Related work}
Various machine learning methods have been proposed for distribution data, ranging from
distribution-to-real regression \citep{poczos2013distribution,oliva2014fast} to
distribution-to-distribution regression~\citep{oliva2015fast,oliva2013distribution}. The
Triple-Basis Estimator (3BE) has been proposed for function-to-function regression. It uses basis
representations of functions and learns a mapping from Random Kitchen Sink basis features
\citep{oliva2015fast}. 3BE shows improved accuracies for distribution regression compared to an
instance-based learning method \citep{oliva2013distribution}. More recently,
\citet{KOU2018}  proposed the distribution regression network which extends the neural network
representation such that an entire distribution is encoded in a single node. With this compact
representation, DRN showed better accuracies while using much fewer parameters than conventional
neural networks and 3BE \citep{oliva2015fast}.

For predicting the future state of a time-varying distribution,
\citet{lampert2015predicting}  proposed Extrapolating the Distribution Dynamics (EDD) which predicts
the future state of a time-varying distribution given samples from previous time steps. EDD uses the
reproducing kernel Hilbert space (RKHS) embedding of distributions and learns a linear mapping to
model how the distribution evolves between adjacent time steps. EDD is shown to work for a few
variants of synthetic data, but the performance deteriorates for tasks where the dynamics is
non-linear.

\section{Distribution Regression Network} \label{sect:DRN}
For the distribution regression task, the dataset consists of $M$ data points $\mathcal{D} = \{(X_1^1,\cdots,X_1^K,Y_1), \cdots,
(X_M^1,\cdots,X_M^K,Y_M)\}$ where $X_i^k$ and $Y_i$ are univariate continuous distributions with
compact support. The regression task is to learn the function $f$ which maps the input distributions
to output distribution: $Y_i=f(X_i^1,\cdots,X_i^K)$ on unseen test data.

\begin{figure}
\centering
\begin{subfigure}[b]{0.38\columnwidth}
\includegraphics[width=\columnwidth]{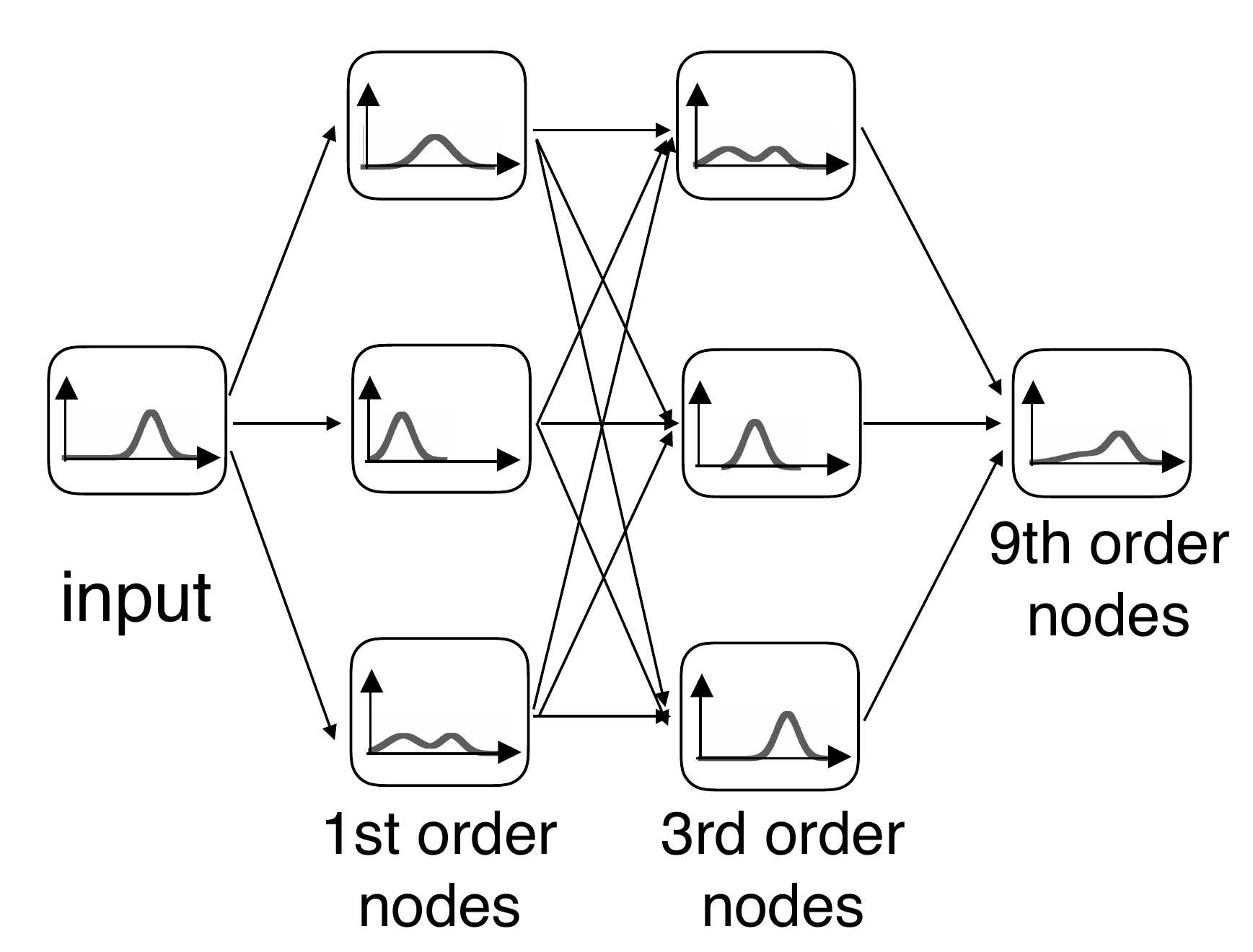}
\caption{}
\label{fig:drnhiddenhidden}
\end{subfigure}
\begin{subfigure}[b]{0.55\columnwidth}
\includegraphics[width=\columnwidth]{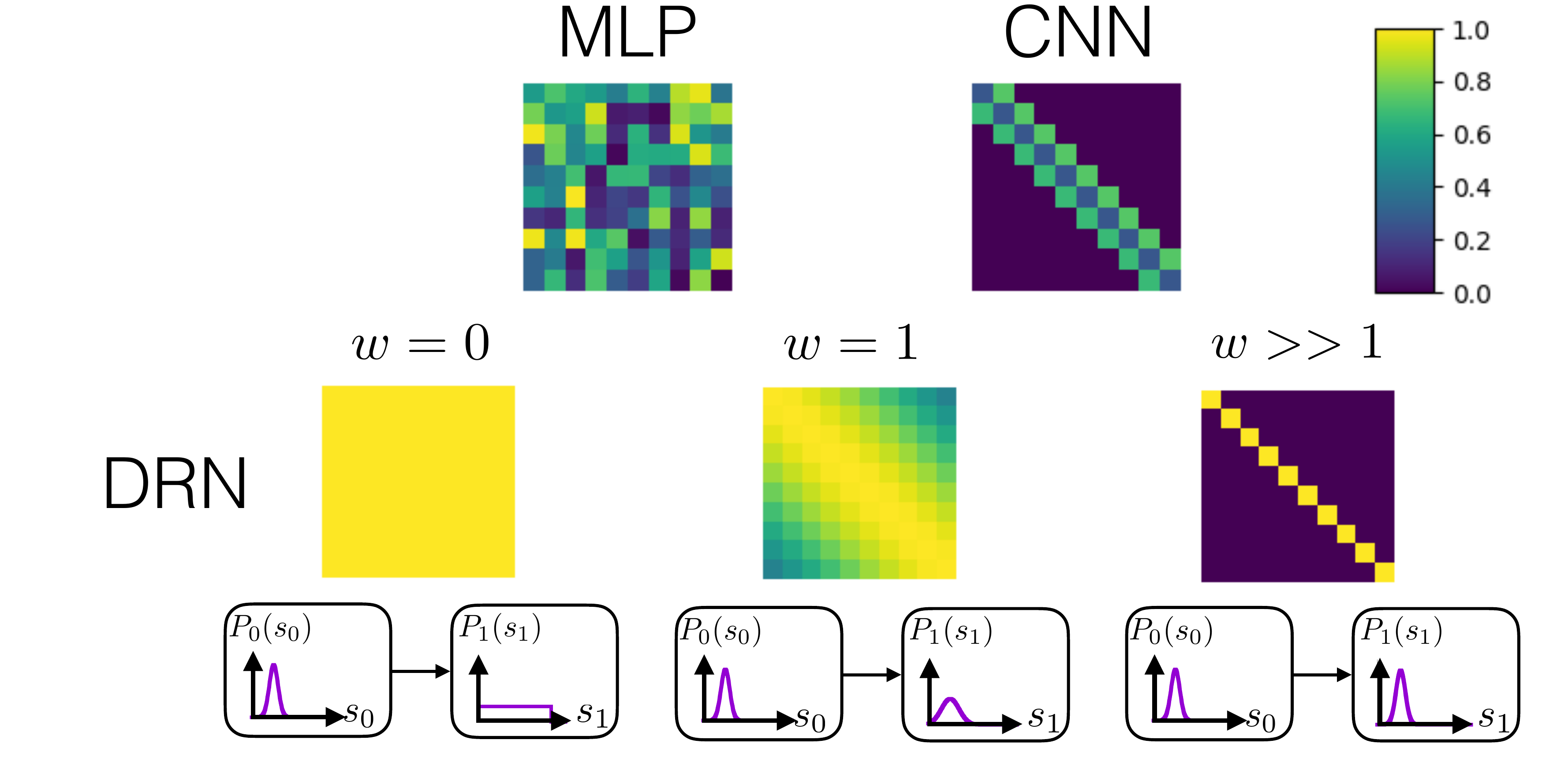}
\caption{}
\label{fig:visualTw}
\end{subfigure}
\caption{(a) DRN performs distribution regression by encoding each node with an entire distribution. Using multiple hidden layers, the DRN network shown has an output up with up to $9^{th}$ order
transformation. (b) Visualization of the weight matrices for DRN, MLP and CNN where $q=10$. For DRN, we show different values of weight, along with the propagation
behaviors.}
\end{figure}

\citet{KOU2018} proposed the distribution regression network~(DRN) for the task of
distribution-to-distribution regression. We give a brief description of DRN following the notations
of \citet{KOU2018}. Figure \ref{fig:drnhiddenhidden} shows that DRN consists of
multiple fully-connected layers connecting the data input to the output in a feedforward manner,
where each connection has a real-valued weight. The novelty of DRN is that each node in the network
encodes a univariate probability distribution. The distribution at each node is computed using the
distributions of the incoming nodes, the weights and the bias parameters. Let $P_k^{(l)}$
represent the probability density function (pdf) of the $k^{\text{th}}$ node in the
$l^{\text{th}}$ layer where $P_k^{(l)}(s_k^{(l)})$ is the density of the distribution when the node
variable is $s_k^{(l)}$. The unnormalized distribution $\tilde{P}_k^{(l)}$ is computed by
marginalizing over the product of the unnormalized conditional probability and the incoming
probabilities.

\begin{align}
\label{eq:margin}
\tilde{P}_k^{(l)}\left(s_k^{(l)}\right) = \int_{\mathbf{s}^{(l-1)}} 
\tilde{Q}(s_k^{(l)}|\mathbf{s}^{(l-1)}) P_1^{(l-1)}\left(s_1^{(l-1)}\right)\cdots
P_n^{(l-1)}\left(s_n^{(l-1)}\right)  \,d\mathbf{s}^{(l-1)},
\end{align}

where the shorthand $\mathbf{s}^{(l-1)} =s_1^{(l-1)}  \cdots s_n^{(l-1)}$ is used for the incoming
node variables and $\tilde{Q}(s_k^{(l)}|\mathbf{s}^{(l-1)}) = \exp[
-E(s_k^{(l)}|s_1^{(l-1)},\cdots,s_n^{(l-1)})]$. The unnormalized conditional probability
has the form of the Boltzmann distribution, where $E$ is the energy for a given set of node
variables,
\begin{align}
\label{eq:energy}
E\left(s_k^{(l)}|\mathbf{s}^{(l-1)} \right) 
= \sum_i^n w_{ki}^{(l)}
\left( \frac{s_k^{(l)}-s_i^{(l-1)}}{\Delta}\right)^2 
+ b_{q,k}^{(l)}\left( \frac{s_k^{(l)}-\lambda_{q,k}^{(l)}}{\Delta} \right)^2 +
b_{a,k}^{(l)} \left|\frac{s_k^{(l)}-\lambda_{a,k}^{(l)}}{\Delta} \right|,
\end{align}
where $w_{ki}^{(l)}$ is the weight connecting the $i^\text{th}$ node in layer $l-1$ to the
$k^{th}$ node in layer $l$.  $b_{q,k}^{(l)}$ and $b_{a,k}^{(l)}$ are the quadratic and absolute bias
terms acting on positions $\lambda_{q,k}^{(l)}$ and $\lambda_{a,k}^{(l)}$ respectively. $\Delta$ is
the support length of the distribution. After obtaining the unnormalized probability, the
distribution from Eq. (\ref{eq:margin}) is normalized. Forward propagation is performed layer-wise
to obtain the output prediction. The DRN propagation model, with the Boltzmann distribution, is
motivated by work on spin models in statistical physics \citep{katsura1962statistical,
lee2002monte}. The distribution regression task is general and in this paper we extend it to the task of forward
prediction on a time-varying distribution: Given a series of distributions with $T$ equally-spaced
time steps, $X^{(1)}, X^{(2)},\cdots, X^{(T)}$, we want to predict $X^{(T+k)}$, i.e.~the distribution
at $k$ time steps later. The input at each time step may consist of more than one distribution. To
use DRN for the time series distribution regression, the input distributions for all time steps are
concatenated at the input layer. The DRN framework is flexible and can be extended to
architectures beyond feedforward networks. Since we are addressing the task of time series
distribution regression, we also implemented a recurrent extension of DRN, which we call recurrent
distribution regression network (RDRN). The extension is straightforward and we provide details of
the architecture in the supplementary material.

Following \citet{KOU2018}, the cost function for the prediction task is measured by the
Jensen-Shannon divergence \citep{lin1991divergence} between the label and
predicted distributions. We adopt the same parameter initialization method as
\citet{KOU2018}, where the network weights and bias are randomly initialized following a uniform
distribution and the bias positions are uniformly sampled from the support length of the data
distribution. The integrals in Eq. (\ref{eq:margin}) are performed numerically. Each continuous
distribution density function is discretized into $q$ bins, resulting in a discrete
probability mass function (i.e. a $q$-dimensional vector that sums to one).

\section{Properties of DRN}\label{sect:DRNanalysis}
DRN is able to perform transformations such as peak spreading and peak splitting, as discussed in detail in \citet{KOU2018}. In this section, we provide further theoretical analysis on the functional form of DRN propagation which
gives more insight to DRN's generalization abilities. By
expressing the integral in Eq. (\ref{eq:margin}) with summation and the distribution at each node as
a discretized $q$-length vector, we can express the output at a node in vector form,
$
\tilde{p}_0 = B_0 \circ (T_{w_1} \cdot p_1) \circ (T_{w_2} \cdot p_2)
\circ \cdots \circ (T_{w_n}\cdot p_n) = B_0 \circ \breve{\prod}^n_{i=1}
T_{w_i} \cdot p_i,
$
where $\breve{\prod}$ is a symbol for Hadamard products and upon normalization, $p_0 =
\tilde{p}_0 / | \tilde{p}_0|$. $\circ$ is the element wise Hadamard product operator and $\cdot$ is
the matrix multiplication operator. $B_0$ is a vector representing the bias term 
whose components are given by
$
(B_0)_i = \exp\left( 
-b_q \left(\frac{s_i - \lambda_q}{\Delta}\right)^2  
-b_a \left|\frac{s_i - \lambda_a}{\Delta}\right| \right)
$.
$T_{w_i}$ is a symmetric  $q$ by $q$ transformation matrix corresponding to the 
connections in DRN with elements
$
(T_{w_i})_{qr} = \exp\left( - w_i 
\left(\frac{s_q - s_r}{\Delta}\right)^2 \right)
$.
With the Boltzmann distribution, under positive weight $w$, the matrix $T_{w_i}$ acts as a Gaussian
filter on the input distribution, where $w$ controls the spread, as shown in Figure
\ref{fig:visualTw}. In the following, we present some propositions concerning DRN,
where their corresponding proofs are in the supplementary material.

\begin{proposition}
A node connecting to a target node with zero weight $w=0$ has no effect on the
activation of the target node.
\end{proposition}
Similar to conventional neural networks, this is a mechanism for which DRN
can learn to ignore spurious nodes by setting their weights to zero or near zero.
\begin{proposition}
For a node connecting to a target node with sufficiently large positive weight $w \to \infty$, the
transformation matrix $T_w \to I$. 
\end{proposition}
The consequence is that the identity mapping from one node to another can be realized.
\begin{proposition}
Output of DRN is invariant to normalization of all hidden layers of DRN.
\label{prop:hidden_norm}
\end{proposition}
As an extension to Proposition \ref{prop:hidden_norm}, it can be shown that the output of DRN is
invariant to scaling of the hidden layers. Therefore arbitrary scaling can be applied in the layers
to control numerical stability. These normalization can be done dynamically during the computation.
In this paper, we found that normalization of all layers is sufficient to provide the required
numerical stability and precision for all our datasets.

\begin{definition}
A node in DRN is said to be an order $n$ node when it is connected with non-zero weights from $n$
incoming nodes in the previous layer.
\label{def:ordern}
\end{definition}
\begin{lemma}
For an order $n$ node, components of $\tilde{p}_{0}$ 
(which we denote as $\tilde{p}_{0i}$),
follow a power law of
$n$th order cross terms of the components of
connecting nodes.
\end{lemma}
\begin{eqnarray} 
\tilde{p}_{0i} &=& 
(B_0)_i 
\sum_{j_1} \cdots \sum_{j_n}
\underbrace{\left[ (T_{w_1})_{i,j_1} \cdots (T_{w_n})_{i,j_n} \right]}_\text{coefficients}
\underbrace{\left[ (p_1)_{j_1} \cdots (p_n)_{j_n} \right]}_\text{cross terms} 
\label{eq:cross}
\end{eqnarray}

Writing in short hand notation, $\jid{1} = (j_1,\cdots j_n)$ where
the superscript indicates $J$ is the indices over the first layer.
Write
$\sum_{j_1} \cdots \sum_{j_n} = \sum_{\jid{1}}$ and consolidate the 
coefficients into a tensor, 
$\cid{i,\jid{1}} (w, B) = (B_0)_i (T_{w_1})_{i,j_1} \cdots (T_{w_n})_{i,j_n}$, and the cross terms into a tensor, 
$P_{\jid{1}} = (p_1)_{j_1} \cdots (p_n)_{j_n}$, where $w = (w_1,\cdots w_n)$.
Eq. (\ref{eq:cross}) can be written compactly as,
\begin{equation}
p_{0} = \sum_{\jid{1}} \cid{\jid{1}} (w, B)  P_{\jid{1}} /
\sum_{\jhd{1}} \zid{\jhd{1}} (w, B)  P_{\jhd{1}}
\label{eq:p2}
\end{equation}
where $p_{0} = (p_{01},p_{02},\cdots p_{0q})$, 
$\cid{\jid{1}} = (\cid{1,\jid{1}},\cdots \cid{q,\jid{1}})$, and $\zid{\jhd{1}}(w, B)  = \sum_i
\cid{i,\jid{1}}(w, B) $.

Using Proposition \ref{prop:hidden_norm}, we consider unnormalized hidden layers. The activations
for the $\alpha^{th}$ node of the first hidden layer $\hid{1}{\alpha}$ and second hidden layer
$\hid{2}{\alpha}$ are,
$
\hid{1}{\alpha} = \sum_{\jidd{1}{\alpha}} 
\cid{\jidd{1}{\alpha}} P_{\jidd{1}{\alpha}}
$, 

$
\hid{2}{\alpha} = \sum_{\jidd{2}{\alpha}} 
\cid{\jidd{2}{\alpha}} H_{\jidd{2}{\alpha}}
$ and 
$
H_{\jidd{2}{\beta}} =
\prod_{\beta=1}^{n_1}
\sum_{\jidd{1}{\beta}} 
\cid{j_{\beta},\jidd{1}{\beta}} P_{\jidd{1}{\beta}} 
$.
Detailed derivations of the above equations are given in the 
supplementary material. Each of the $P_{\jidd{1}{\beta}}$ consists of cross terms of the input
distributions to order
$n_0$ ($n_0$ is the number of input nodes).
$H_{\jidd{2}{\alpha}}$ is a product of
$n_1$ terms of $P_{\jidd{1}{\beta}}$'s, hence
$H_{\jidd{2}{\alpha}}$ will be cross terms of the input distributions to order
$n_1 \times n_0$. For a network of $L$ hidden layers with number of nodes, $n_1,n_2, \cdots n_L$,
the output consist of multiplications of the components of input distributions to
the power of $n_0 \times n_1 \cdots \times n_L$. In this way, DRN can fit high order functions
exponentially quickly by adding hidden layers. For instance, the network in Figure~\ref{fig:drnhiddenhidden} obtains a $9^{th}$ order transformation by using two hidden layers of only
3 nodes each.

\begin{proposition}
For a node of order $n$, in the limit of small weights $|w_\alpha| \ll 1$
for $\alpha=1,\cdots n$, the output activions, ${p}_{0}$
can be approximated as a fraction of two
linear combinations of the activations in the input nodes. 
\label{prop:linear}
\end{proposition}
The consequence of proposition \ref{prop:linear} is that by adjusting
the weights and by expanding the matrix $T_{w_\alpha}$ to orders linear in $w_\alpha$,
DRN can approximate the output distribution to be a fraction of linear
combinations of the input distributions in the form,
\begin{equation}
p_{0i} \approx \frac{B_{0i} + 
	\sum_\alpha B_{0i} \circ (\mathcal{E}(w_\alpha) \cdot p_\alpha)_i }
{\sum_j  \left[ B_{0j} + \sum_{\alpha'} B_{0j} \circ (\mathcal{E}(w_{\alpha'}) \cdot p_{\alpha'})_j \right]}
\end{equation}
$\mathcal{E}$ is a matrix linear in $w_\alpha$.
Indeed, the matrix $T_{w_\alpha}$ can be approximated by expanding
to $K$ orders in $w_\alpha$ 
with accuracy of expansion depending on  the magnitudes of $w_\alpha$. If
expansion is up to second order in $w_\alpha$ 
then the output is a fraction of quadratic expressions. If
the expansion in $w_\alpha$ is up to $K$ order then the resulting output is a fraction of polynomials of
$K$ order. At this point we wish to mention DRN's analogy to the well known Pad\'e approximant
\citep{baker1996pade}. Pad\'e approximant is a method of function approximation using fraction of
polynomials.

We compare the linear transformations of DRN with MLP and CNN for the case of transforming an input
distribution to an (unnormalized) output distribution using one network layer. MLP consists of a
linear transformation with a weight matrix followed by addition of a bias vector and elementwise
transformation with an activation function. The linear transformation can be expressed as
$\mathbf{\tilde{p}}_1^{\text{MLP}} =  W \mathbf{p}_0 + \mathbf{b}$, where $\mathbf{p}_0$ is a
$q$-length vector, and $\mathbf{\tilde{p}}_1^{\text{MLP}} $ is the corresponding unnormalized output
distribution, $W$ is a dense weight matrix with $q \times q$ parameters and $\mathbf{b}$ has $q$
values. For CNN, the one-dimensional convolutional filter acts on the input distribution with the weight matrix  $W_c$ arising from a convolutional filter. We compare the linear transformations with the illustration in Figure~\ref{fig:visualTw}.
DRN's weight matrix is highly regularized, where the single free parameter $w$ controls the
propagation behavior. In contrast, MLP has a dense weight with $q \times q$ parameters. The weight
matrix in CNN is more regularized than in MLP but uses more free parameters than DRN. As a result of
DRN's compact representation of the distribution and regularized weight matrix, interpretation of
the network by analyzing the weights becomes easier for DRN.

For non linear transformation, MLP uses a non-linear activation function and with
the single hidden layer,
MLP can fit a wide range of non-linear functions with sufficient number
of hidden nodes
\citep{lecun2015deep}. The non-linear transformation in CNN is similar to MLP, except that the
convolutional filters act as a regularized transformation. 
Without the hidden layer, MLP and CNN behave like logistic regression and they
can only fit functions with linear level-sets.
In contrast, DRN has no activation
function, it achieves nonlinear transformations by using the Hadamard product as explained in
Definition \ref{def:ordern}. As a consequence of the Hadamard product, with no hidden layers, DRN
can fit functions with non-linear level-sets.

In this section we have provided theoretical analysis on DRN's propagation and showed how the varied
propagation behaviors can be controlled by just a few network parameters in contrast to MLP and
CNN. In the subsequent experiments, we show that DRN consistently uses fewer model parameters than
tranditional neural networks and 3BE while achieving better test accuracies. We further investigate
the generalization capablities by varying the number of training data and number of samples drawn
from the distributions.

\section{Experiments}
We conducted experiments with DRN and RDRN on four datasets which involve prediction of time-varying
distributions. We also compare with conventional neural network architectures and other distribution
regression methods. The benchmark methods are multilayer perceptron (MLP), recurrent neural network
(RNN) and Triple-Basis Estimator (3BE) \citep{oliva2015fast}. For the third dataset, we also compare
with Extrapolating the Distribution Dynamics (EDD) \citep{lampert2015predicting} as the data
involves only a single trajectory of distribution. Among these methods, RDRN, RNN and EDD are
designed to take in the inputs sequentially over time while for the rest the inputs from all $T$
time steps are concatenated. Each distribution is discretized into $q$ bins. 

\begin{figure}
\begin{subtable}[b]{0.6\textwidth}
	\small
	\centering
\caption{}
	\label{table:res1}
\begin{tabular}{@{}ccc|cc@{}} \hline
	& \multicolumn{2}{c|}{Shifting Gaussian} & \multicolumn{2}{c}{Climate Model}    \\
	& \multicolumn{2}{c|}{(20 training data)} & \multicolumn{2}{c}{(100 training data)}    \\
	& Test $L_2 (10^{-2})$  & $N_{p}$  & Test $L_2 (10^{-2})$      & $N_{p}$      \\  \hline
	DRN    & \textbf{4.90(0.46)}  & 224   & \textbf{12.27(0.34)} & 44   \\
	RDRN  & \textbf{4.55(0.42)}   & 59   & \textbf{11.98(0.13)}   & 59  \\
	MLP    & 10.32(0.41) & 1303   & 13.52(0.25) & 22700   \\
	RNN    & 17.50(0.89)   & 2210   & 13.29(0.59) & 12650    \\
	3BE    & 22.30(1.89)   & 6e+5   & 14.18(1.29) & 2.2e+5  \\ \hline
\end{tabular}
\end{subtable}
	\begin{subfigure}[c]{0.38\textwidth}
			\caption{}
						\label{fig:sg_Ntrain}
			\includegraphics[width=\columnwidth]{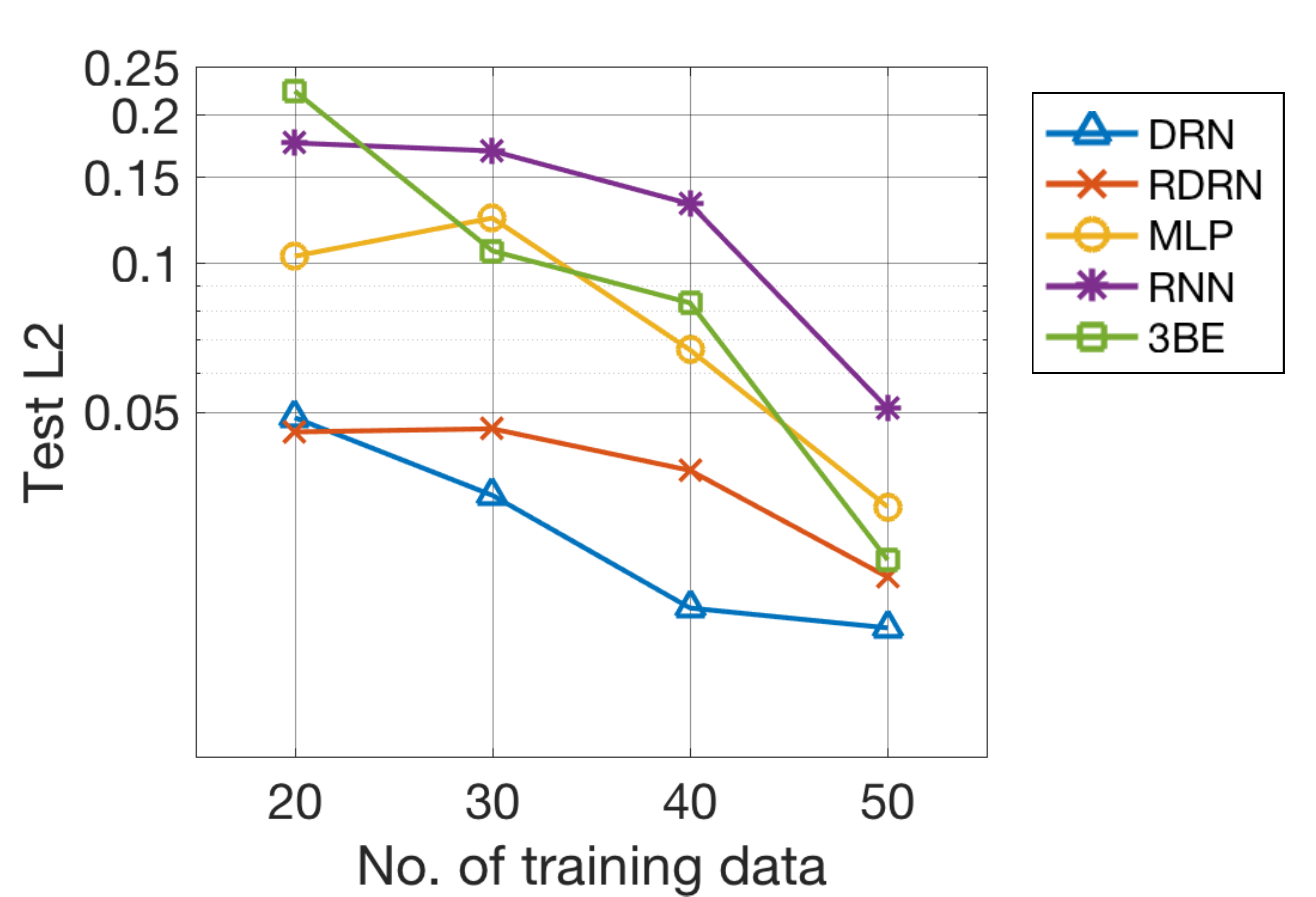}
	\end{subfigure}
\caption{(a) Regression results for the shifting Gaussian and climate model datasets. $N_{p}$: number of model
	parameters. The number in the parentheses is
	the standard error over repeated runs. (b)~Shifting Gaussian dataset: The test performance at small training sizes. Best viewed incolor.}
\end{figure}

\subsection{Shifting Gaussian}
For the first experiment, we adapted the shifting Gaussian experiment from
\citet{lampert2015predicting} but made it more challenging. 
Although this is a synthetic data set, this is the most challenging of all our datasets
as it involves complex shifts in the distribution peaks.
It is used to empirically study how DRN performs better even with fewer number of 
parameters as compared to the benchmark methods.
Our shifting Gaussian means varies sinusoidally over time. Given a few consecutive input distributions taken from time steps spaced
$\Delta t = 0.2$ apart, we predict the next time step distribution. Because of the sinusoidal variation, it is apparent that we require
more than one time step of past distributions to predict the future distribution. 
The specific details of the data construction is in supplementary material. We found that for all methods, a history length of 3 time steps is
optimal. Following \citet{oliva2014fast} the regression performance is measured by the $L_2$ loss,
where lower $L_2$ loss is favorable. 

The plots in Figure \ref{fig:sg_Ntrain} show the test $L_2$ loss as the number of training data
varies. Across the varying training sizes, DRN and RDRN outperform the other methods, except at
training size of 50, where 3BE's test performance catches up. As training size decreases, MLP, RNN
and 3BE show larger decrease in test performance. DRN performs better than RDRN, except at the
smallest training size of 20 where there is no significant difference. The table in Figure \ref{table:res1} shows
the regression results for the training size of 20, and we note DRN and RDRN use much fewer
parameters than the other methods. Overall, DRN and RDRN require at least two times fewer training
data than the other methods for similar test accuracies.

\begin{figure}[h!]
	\begin{subfigure}[b]{0.35\columnwidth}
		\includegraphics[width=\linewidth]{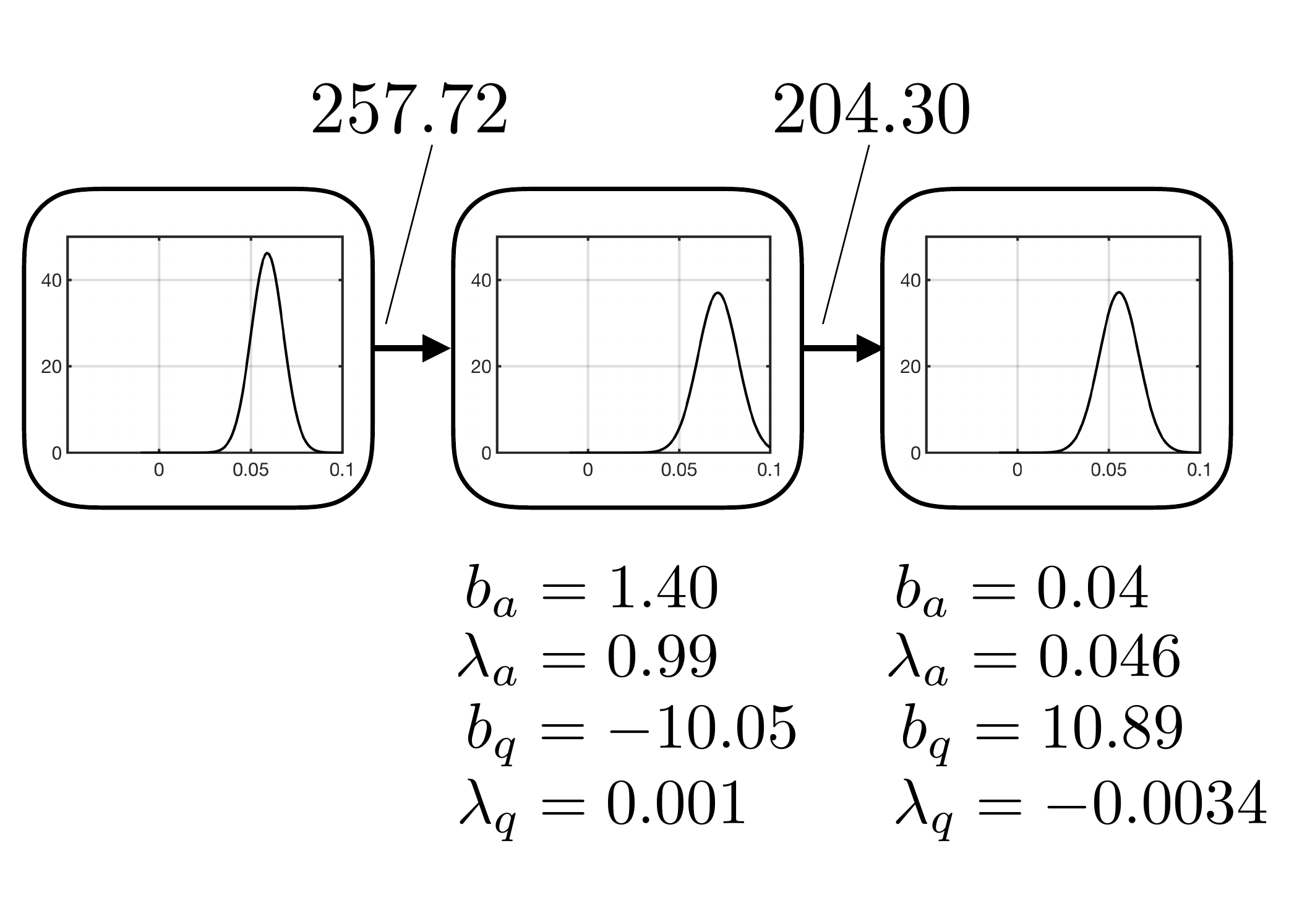}
		\caption{}
		\label{fig:drnvisualizeou}
	\end{subfigure}
\captionsetup{justification=centering}
\begin{subfigure}[b]{0.28\columnwidth}
\includegraphics[width=\columnwidth]{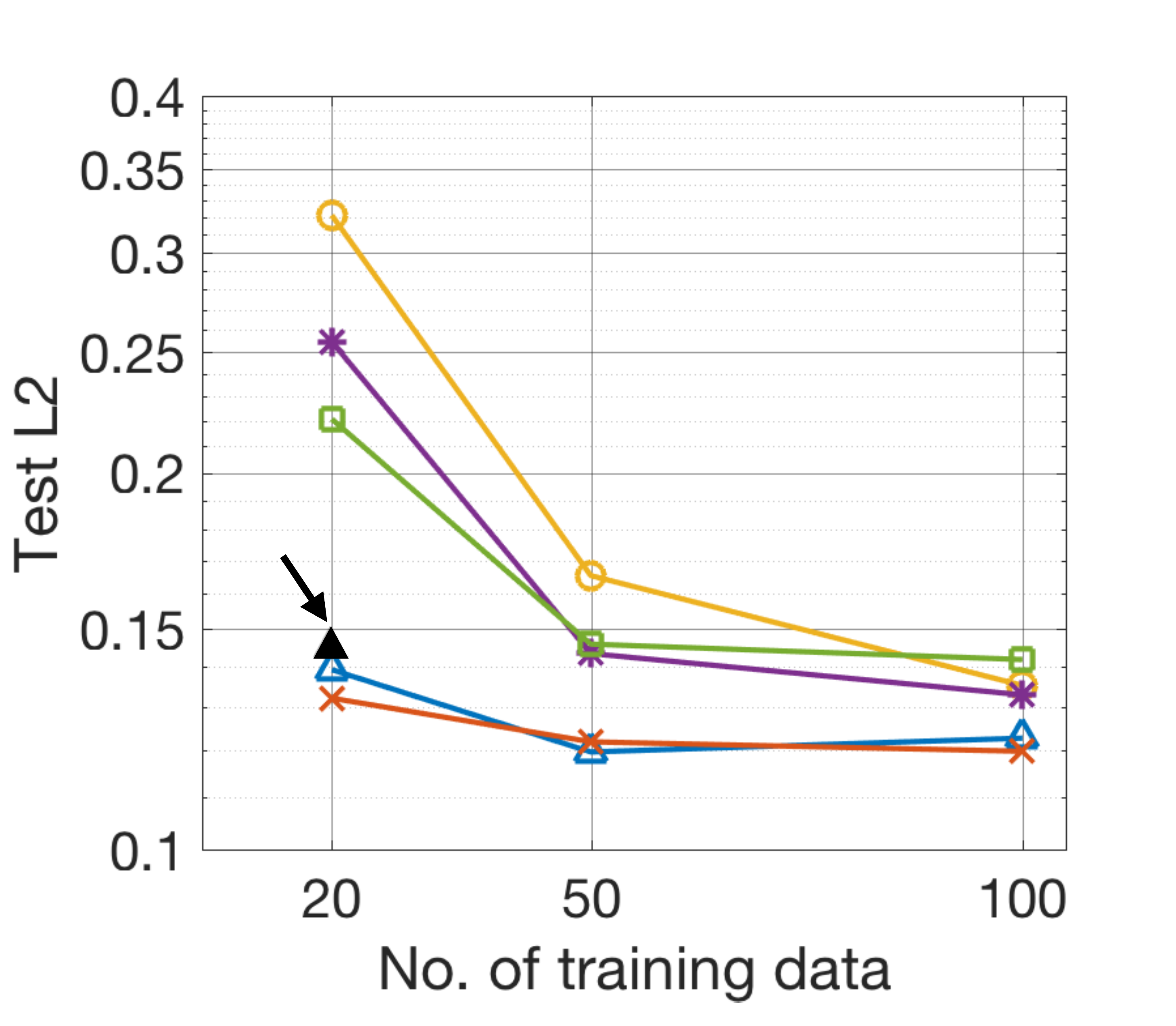}
\caption{Varying number of training data \\ (full pdf, without sampling)}
\label{fig:OU_Ntrain}
\end{subfigure}
\begin{subfigure}[b]{0.35\columnwidth}
\includegraphics[width=\columnwidth]{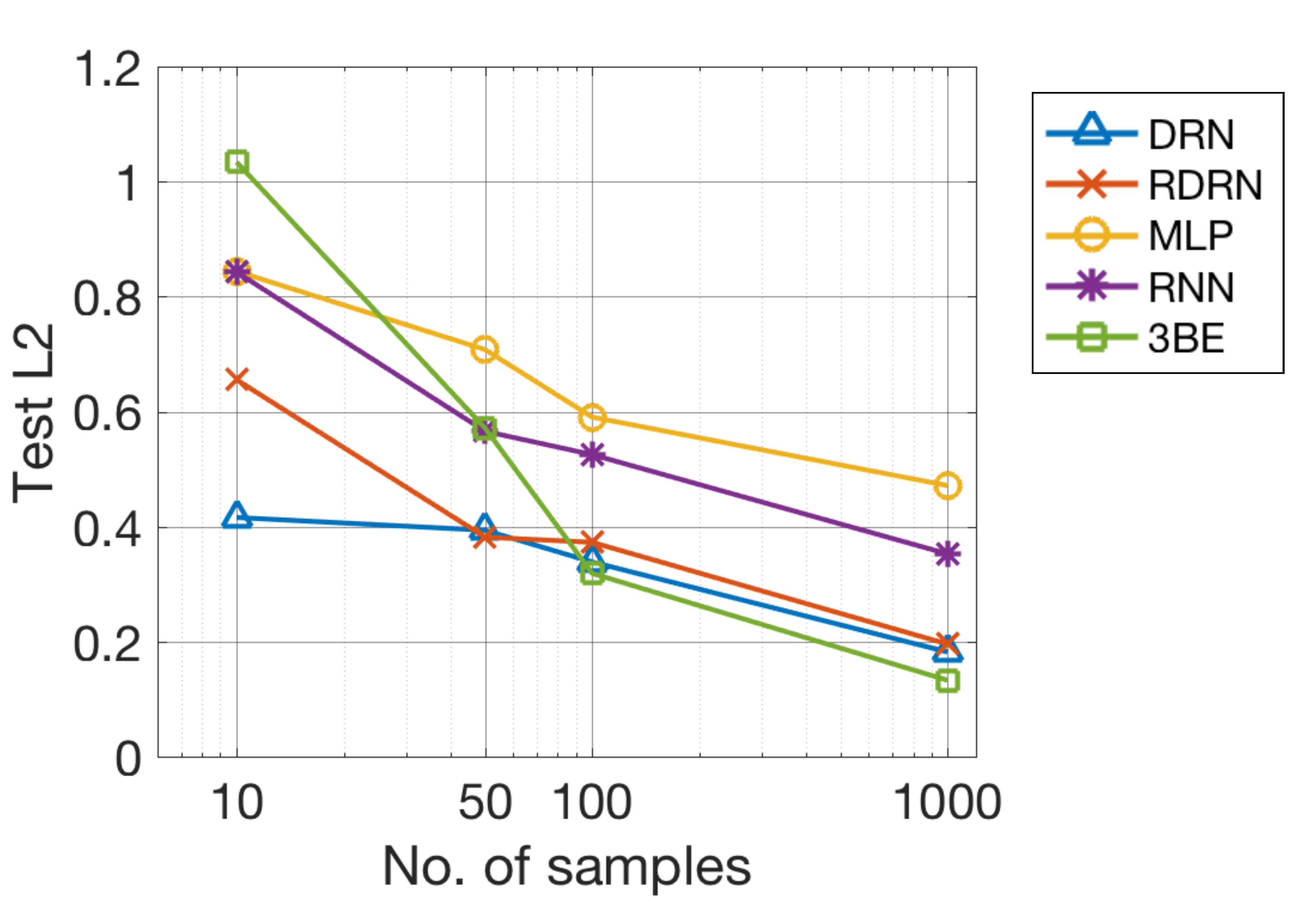}
\caption{Varying number of samples per distribution \\(100 training data)}
\label{fig:OU_Nsamples}
\end{subfigure}
\captionsetup{justification=raggedright}
\caption{Climate model dataset: (a) A DRN network with only 2 weight and 8 bias parameters performs well for the climate
	model dataset. (b) The accuracies of DRN and RDRN remain high as number of training
data decreases. The black triangle indicated by the arrow gives the performance obtained by the DRN network in (a). (c) As number of samples decreases, DRN and RDRN remain robust while the other
methods saw larger decrease in accuracy.}
\label{fig:OUres}
\end{figure}

\subsection{Climate model}
In this experiment, we test the ability of the regression methods to model unimodal Gaussian distributions spreading and drifting over time. We use a climate model which predicts how heat flux at the  sea surface varies \citep{lin1987nonlinear}. The evolution of the heat flux obeys the Ornstein-Uhlenbeck (OU) process
which describes how a unimodal Gaussian distribution spreads and drifts over time. 
The regression task is as follows: Given a sequence of
distributions spaced equally apart with some fixed time gap, predict the distribution some fixed
time step after the final distribution. Details of the dataset is in the supplementary
 material.

The regression results on the test set are shown in the table in Figure \ref{table:res1}. The regression
accuracies for DRN and RDRN are the best, followed by MLP and RNN. With limited training data, the
performance of DRN and RDRN remain very good whereas the other methods saw large decrease in
performance (see Figure ~\ref{fig:OU_Ntrain}). 
Next, we study how the sampling noise in data affects
the regression models' performance by generating different numbers of samples drawn
from distribution. Figure
\ref{fig:OU_Nsamples} shows the performance for varying number of samples drawn for a training size
of 100. DRN and RDRN remain robust at large sampling noise whereas the other methods saw larger
increase in error. 

DRN's compact representation of distributions allows it to perform transformations with very few
network parameters, as discussed in Section \ref{sect:DRNanalysis}. 
In Figure
\ref{fig:drnvisualizeou}, we show a DRN network that has good test accuracy on the climate model
dataset. The network consists of just one hidden node in between the input and output nodes and has
only 10 parameters. We observe that the output distribution follows the expected behavior for the
climate model: shifting towards the long-term mean at zero, with some spread from the input
distribution. The DRN network in Figure \ref{fig:drnvisualizeou} does this by first shifting the
distribution right, and then left, with additional spread at each step.

\begin{table}
	\centering
	\begin{subtable}{.4\linewidth}
		\small
		\begin{tabular}{@{}ccc@{}}  \hline	
			&  \multicolumn{2}{c}{CarEvolution (5 training data)}   \\
			& Test NLL   & $N_{p}$  \\  \hline	
			DRN &  \textbf{3.9663(2e-5)}  & 28676   \\
			RDRN  & \textbf{3.9660(3e-4)}  & 12313   \\
			MLP & 3.9702(6e-4) & 1.2e+7   \\
			3BE & 3.9781(0.003)  & 1.2e+7  \\
			EDD & 4.0405  & 64 \\  \hline
		\end{tabular}
		\caption{}
		\label{table:res_cars}
	\end{subtable}%
	\begin{subtable}{.55\linewidth}
		\small
		\begin{tabular}{@{}ccccc@{}}  \hline	
			&  \multicolumn{4}{c}{Stock (200 training data)}   \\
			&  Test NLL (1 day)     & Test NLL (10 days)  & $T$ & $N_{p}$  \\  \hline	
			DRN  & \textbf{-473.93(0.02)} & -458.08(0.01) & 1 & 9    \\
			RDRN & -469.47(2.43) & \textbf{-459.14(0.01)} & 3 & 37    \\
			MLP & -471.00(0.04)& -457.08(0.98) & 3 & 10300   \\
			RNN & -467.37(1.33) & -457.96(0.20)& 3 & 4210   \\
			3BE & -464.22(0.16) & -379.43(11.8) & 1 & 14000   \\ \hline
		\end{tabular}
		\caption{}
		\label{table:res_stock}
	\end{subtable}
	\caption{Regression results for the (a) CarEvolution  and (b) stock dataset. $NLL$: negative log-likelihood, $T$: optimal number of input time steps, $N_{p}$: number of model parameters used. Lower loss values reflect better regression accuracies.}
\end{table}

\subsection{CarEvolution data}
For the next experiment, we use the CarEvolution dataset \citep{rematas2013does} which was used by
\citet{lampert2015predicting} to evaluate EDD's ability to track the distribution drift of image
datasets. This is very useful for training classifiers where the data distribution changes over time. The dataset consists of images of
cars manufactured from different years and from each time period, we obtain a distribution of DeCAF features \citep{donahue2014decaf} of the car images in that period. Here we make the approximation that the DeCAF features
are independent. For this dataset, the distributions can be multimodal and non-Gaussian, as shown in the pdfs in the supplementary material.

The regression task is to predict the next time step distribution of features given the previous $T$
time step distributions. We found $T$=2 to work best for all methods. 
The regression performance is measured by the negative log-likelihood (NLL) of the
test samples following \citet{oliva2013distribution}, where lower NLL is favorable. The regression
results are shown in Table~\ref{table:res_cars}. DRN and RDRN have the best test performance. RNN
had difficulty in optimization possibly due to the high number of input dimensions, so the results
are not presented. EDD has the fewest number of parameters as it assumes the dynamics of the
distribution follows a linear mapping between the RKHS features of consecutive time steps (i.e.
$T$=1). However, as the results show, the EDD model may be too restrictive for this dataset. For
this dataset, since the number of training data is very small, we do not vary the size of the
training set. We also do not vary the sample size since each distribution contains varying number of
samples.

\subsection{Stock prediction}
The next experiment is on stock price distribution prediction which has been studied extensively by
\citet{KOU2018}. We adopt a similar experimental setup and extend to multiple time steps.
Predicting future stock returns distributions has significant value in the real-world setting and
here we test the methods' abilities to perform such a task. Our regression task is as follows: given
the past $T$ days' distribution of returns of constituent companies in FTSE, Dow and Nikkei, predict
the distribution of returns for constituent companies in FTSE $k$ days later. We used 5 years of
daily returns and each distribution is obtained using kernel density estimation with the
constituent companies' returns as samples. In the supplementary material, we show some samples of the distributions.

\begin{figure}[h!]
\centering
\begin{subfigure}[b]{0.3\columnwidth}
\includegraphics[width=\columnwidth]{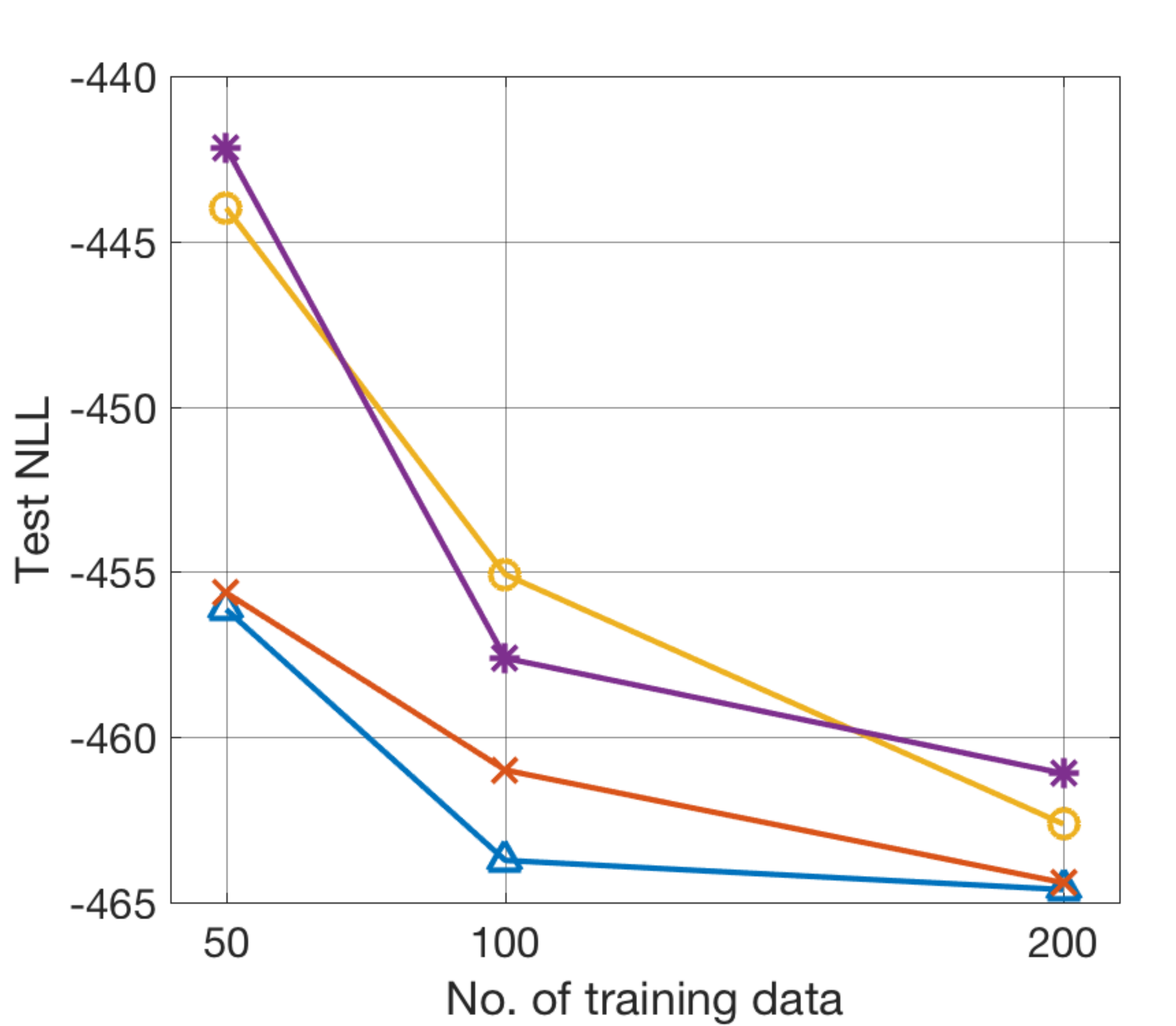}
\caption{5 days ahead}
\label{fig:stock_Ntrain}
\end{subfigure}
\begin{subfigure}[b]{0.3\columnwidth}
\includegraphics[width=\columnwidth]{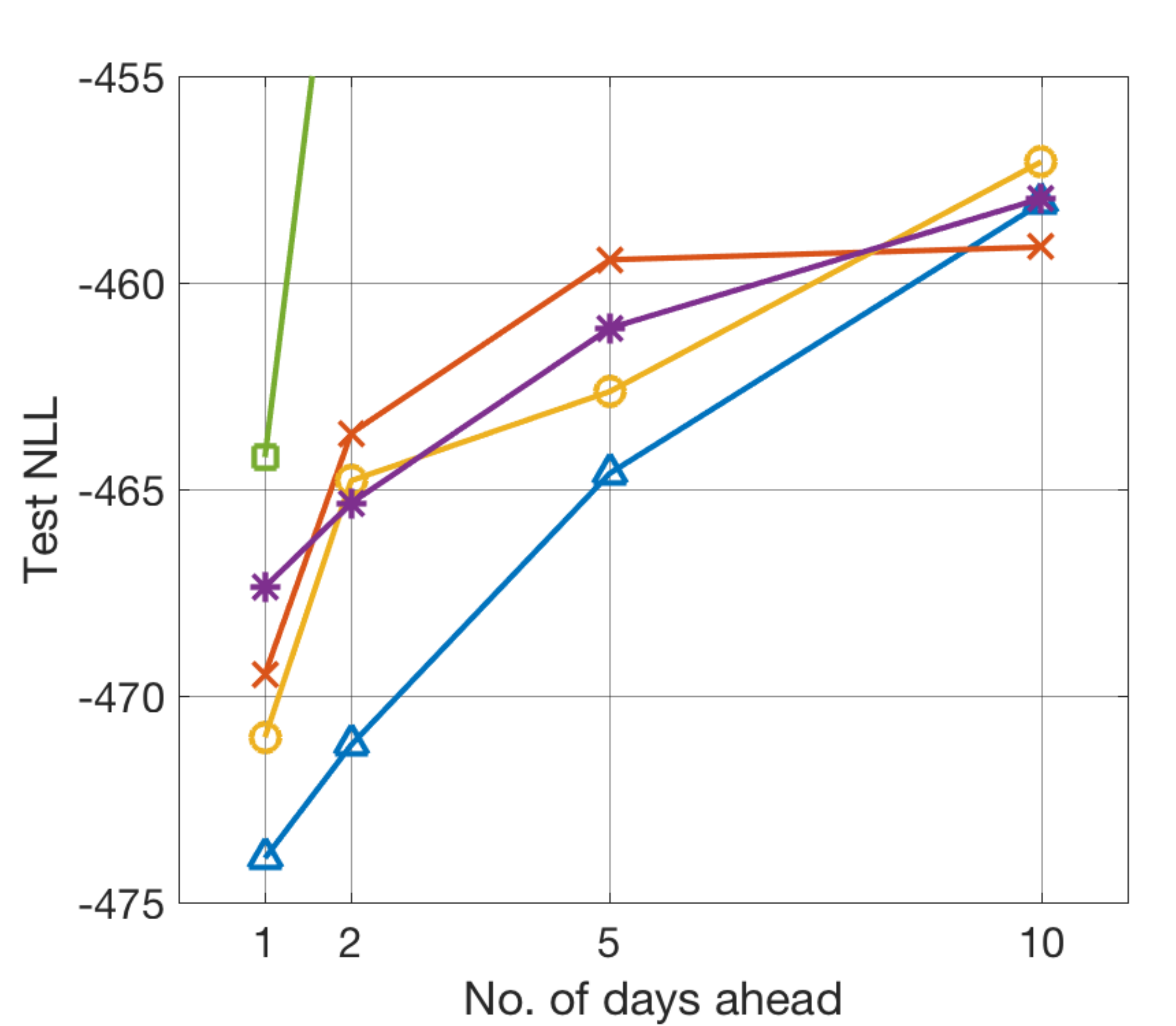}
\caption{200 training data}
\label{fig:stock_Ndays_train200}
\end{subfigure}
\begin{subfigure}[b]{0.38\columnwidth}
\includegraphics[width=\columnwidth]{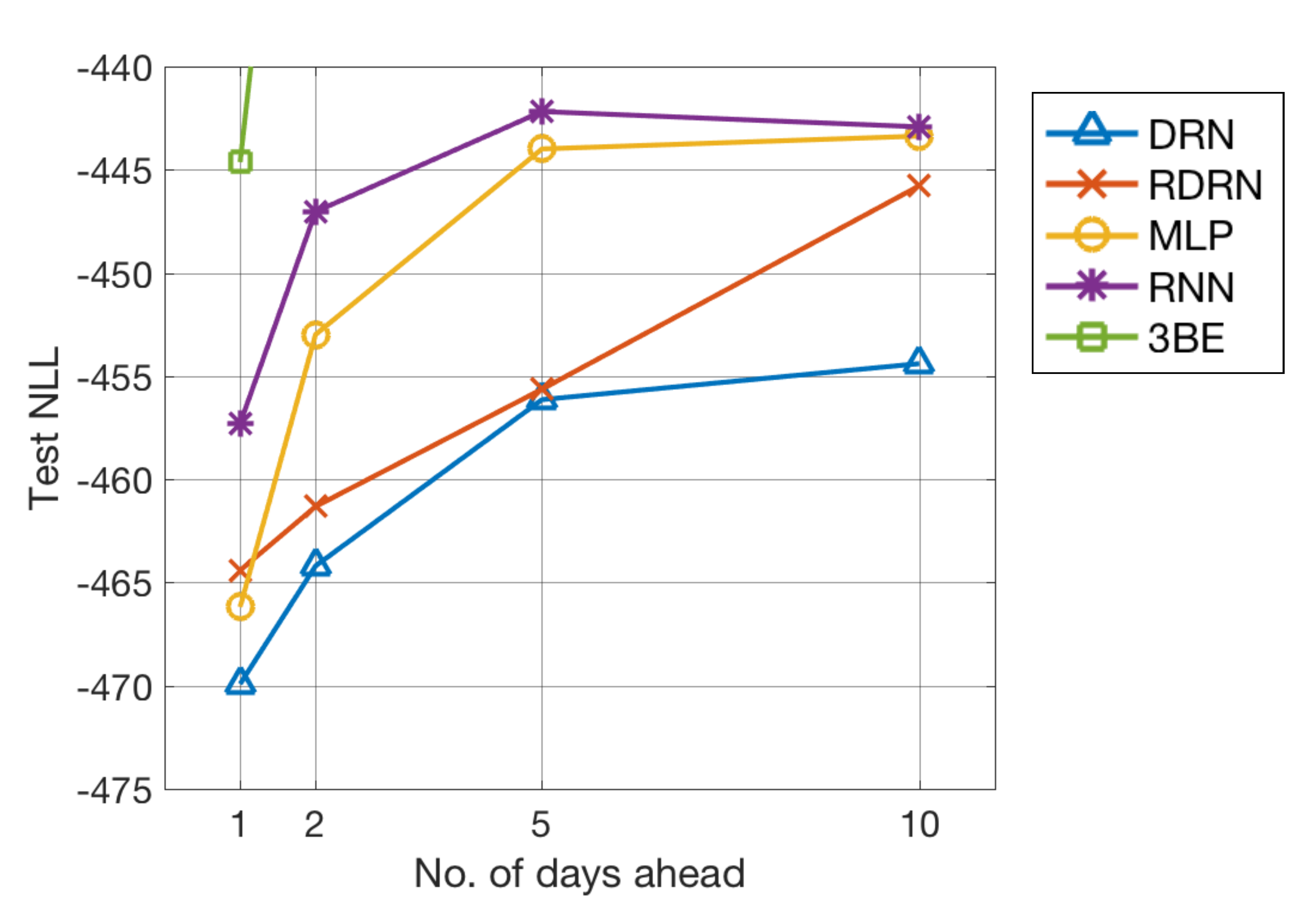}
\caption{50 training data}
\label{fig:stock_Ndays_train50}
\end{subfigure}
\caption{Stock dataset: (a) Test negative log-likelihood for varying number of training data, for 5
days ahead prediction. DRN and RDRN's performance remains robust at low training sizes. 3BE's result
is not shown as it is out of range. (b, c) Test NLL for varying number of days ahead prediction, for
50 and 200 training data. 3BE's NLL is out of range for 2, 5, and 10 days ahead. DRN's performance
remains robust as number of days increases, especially for limited training size of 50. }
\label{fig:stockres}
\end{figure}

We first observe how robust the methods are when training data is limited. Figure
\ref{fig:stock_Ntrain} shows the performance for 5 days ahead prediction with varying number of
training sizes. DRN and RDRN's performance remains robust even as the number of training data
decreases from 200 to 50, whereas MLP, RNN and 3BE have larger decrease in test accuracy. Next, to
study how the models perform with increasing level of task difficulty, we vary number of days
ahead to predict. The results are shown in Figure \ref{fig:stock_Ndays_train200}
and \ref{fig:stock_Ndays_train50} for 200 and 50 training data respectively. As expected, as the
number of days ahead increases, the task difficulty increases and all methods see a decrease in
accuracy. DRN remains robust as the number of days ahead increases, especially for the smaller
training size of 50.

Table \ref{table:res_stock} shows the regression results for training size of 200 for 1 and 10 days
ahead. For 1 day ahead performance, DRN outperforms the other methods, followed by MLP then RDRN. Since for this experiment DRN uses only one previous day of
input, this suggests that the 1 day ahead prediction task does not involve long time dependencies.
Predicting 10 days ahead is a more challenging task which may benefit from having a longer history
of stock movements.  For a training size of 200, RDRN is the best method, using 3 days of input
history. This suggests that for a prediction task which involves longer time dependencies, having a
recurrent architecture for DRN is beneficial when training size is sufficiently large. 

\section{Discussion}
In this work, we address a gap in current work on distribution regression models, in that there is a lack of systematic study on the theoretical basis and generalization abilities of the various methods. The distribution regression network (DRN) has been shown to achieve higher accuracies than conventional neural networks \citep{KOU2018}. To address a lack of theoretical comparison of previous works, we studied the mathematical properties of DRN and conventional neural networks in Section \ref{sect:DRNanalysis}, which gave further insights to the difference in the effects of the network parameters in the various models. In summary, we analyzed that a single weight parameter in DRN can control the propagation behavior in single nodes, ranging from the identity function to peak spreading. The propagation in DRN is highly regularized in contrast to that in MLP and CNN where there are many more parameters. In addition, DRN can fit higher order functions exponentially quickly by adding hidden layers while keeping the network compact.

These mathematical properties of DRN give insights to our experimental findings. We conducted thorough experimental validation on the generalization performance of DRN, conventional neural network models and 3BE. DRN achieves superior test accuracies with robust performance at limited training sizes, noisy data sampling and increasing task difficulty. Furthermore, the number of model parameters in DRN is much smaller. This can be attributed to the mathematical properties of DRN: the highly regularized propagation allows it to generalize better than conventional neural networks.

For future work, we look to extend to multivariate distributions, which will be useful in various applications such as modeling the 3D distribution of dark matter \citep{ravanbakhsh2016estimating} and studying human populations through multi-dimensional census data \citep{flaxman2015supported}. Another possibility is to extend DRN for the general function-to-function regression task.

\bibliography{main}

\begin{thebibliography}{25}
\providecommand{\natexlab}[1]{#1}
\providecommand{\url}[1]{\texttt{#1}}
\expandafter\ifx\csname urlstyle\endcsname\relax
  \providecommand{\doi}[1]{doi: #1}\else
  \providecommand{\doi}{doi: \begingroup \urlstyle{rm}\Url}\fi

\bibitem[Baker et~al.(1996)Baker, Baker~Jr, BAKER~JR, Graves-Morris, and
  Baker]{baker1996pade}
Baker, George~A, Baker~Jr, George~A, BAKER~JR, GEORGE~A, Graves-Morris, Peter,
  and Baker, Susan~S.
\newblock \emph{Pad{\'e} approximants}, volume~59.
\newblock Cambridge University Press, 1996.

\bibitem[Donahue et~al.(2014)Donahue, Jia, Vinyals, Hoffman, Zhang, Tzeng, and
  Darrell]{donahue2014decaf}
Donahue, Jeff, Jia, Yangqing, Vinyals, Oriol, Hoffman, Judy, Zhang, Ning,
  Tzeng, Eric, and Darrell, Trevor.
\newblock Decaf: A deep convolutional activation feature for generic visual
  recognition.
\newblock In \emph{International conference on machine learning}, pp.\
  647--655, 2014.

\bibitem[Flaxman et~al.(2015)Flaxman, Wang, and Smola]{flaxman2015supported}
Flaxman, Seth~R, Wang, Yu-Xiang, and Smola, Alexander~J.
\newblock Who supported obama in 2012?: Ecological inference through
  distribution regression.
\newblock In \emph{Proceedings of the 21th ACM SIGKDD International Conference
  on Knowledge Discovery and Data Mining}, pp.\  289--298. ACM, 2015.

\bibitem[Gu et~al.(2013)Gu, Jeon, and Lin]{gu2013nonparametric}
Gu, Chong, Jeon, Yongho, and Lin, Yi.
\newblock Nonparametric density estimation in high-dimensions.
\newblock \emph{Statistica Sinica}, pp.\  1131--1153, 2013.

\bibitem[Gu{\'e}rin et~al.(2011)Gu{\'e}rin, Prost, and
  Joanny]{guerin2011bidirectional}
Gu{\'e}rin, T, Prost, J, and Joanny, J-F.
\newblock Bidirectional motion of motor assemblies and the weak-noise escape
  problem.
\newblock \emph{Physical Review E}, 84\penalty0 (4):\penalty0 041901, 2011.

\bibitem[Kaastra \& Boyd(1996)Kaastra and Boyd]{kaastra1996designing}
Kaastra, Iebeling and Boyd, Milton.
\newblock Designing a neural network for forecasting financial and economic
  time series.
\newblock \emph{Neurocomputing}, 10\penalty0 (3):\penalty0 215--236, 1996.

\bibitem[Katsura(1962)]{katsura1962statistical}
Katsura, Shigetoshi.
\newblock Statistical mechanics of the anisotropic linear heisenberg model.
\newblock \emph{Physical Review}, 127\penalty0 (5):\penalty0 1508, 1962.

\bibitem[Kou et~al.(2018)Kou, Lee, and Ng]{KOU2018}
Kou, Connie Khor~Li, Lee, Hwee~Kuan, and Ng, Teck~Khim.
\newblock A compact network learning model for distribution regression.
\newblock \emph{Neural Networks}, 2018.
\newblock ISSN 0893-6080.
\newblock \doi{https://doi.org/10.1016/j.neunet.2018.12.007}.
\newblock URL
  \url{http://www.sciencedirect.com/science/article/pii/S0893608018303381}.

\bibitem[Lampert(2015)]{lampert2015predicting}
Lampert, Christoph~H.
\newblock Predicting the future behavior of a time-varying probability
  distribution.
\newblock In \emph{Proceedings of the IEEE Conference on Computer Vision and
  Pattern Recognition}, pp.\  942--950, 2015.

\bibitem[LeCun et~al.(2015)LeCun, Bengio, and Hinton]{lecun2015deep}
LeCun, Yann, Bengio, Yoshua, and Hinton, Geoffrey.
\newblock Deep learning.
\newblock \emph{nature}, 521\penalty0 (7553):\penalty0 436, 2015.

\bibitem[Lee et~al.(2002)Lee, Schulthess, Landau, Brown, Pierce, Gai, Farnan,
  and Shen]{lee2002monte}
Lee, Hwee~Kuan, Schulthess, Thomas~C, Landau, David~P, Brown, Gregory, Pierce,
  John~Philip, Gai, Z, Farnan, GA, and Shen, J.
\newblock Monte {C}arlo simulations of interacting magnetic nanoparticles.
\newblock \emph{Journal of applied physics}, 91\penalty0 (10):\penalty0
  6926--6928, 2002.

\bibitem[Lin \& Koshyk(1987)Lin and Koshyk]{lin1987nonlinear}
Lin, Charles~A and Koshyk, John~N.
\newblock A nonlinear stochastic low-order energy balance climate model.
\newblock \emph{Climate dynamics}, 2\penalty0 (2):\penalty0 101--115, 1987.

\bibitem[Lin(1991)]{lin1991divergence}
Lin, Jianhua.
\newblock Divergence measures based on the {S}hannon entropy.
\newblock \emph{IEEE Transactions on Information theory}, 37\penalty0
  (1):\penalty0 145--151, 1991.

\bibitem[Million(2007)]{million2007hadamard}
Million, Elizabeth.
\newblock The hadamard product.
\newblock \emph{Course Notes}, 3:\penalty0 6, 2007.

\bibitem[Murphy(1999)]{murphy1999technical}
Murphy, John~J.
\newblock Technical analysis of the futures markets: A comprehensive guide to
  trading methods and applications, {New York Institute of Finance}, 1999.

\bibitem[Noble \& Wheatland(2011)Noble and Wheatland]{noble2011modeling}
Noble, PL and Wheatland, MS.
\newblock Modeling the sunspot number distribution with a fokker-planck
  equation.
\newblock \emph{The Astrophysical Journal}, 732\penalty0 (1):\penalty0 5, 2011.

\bibitem[Oliva et~al.(2015)Oliva, Neiswanger, P{\'o}czos, Xing, Trac, Ho, and
  Schneider]{oliva2015fast}
Oliva, Junier, Neiswanger, William, P{\'o}czos, Barnab{\'a}s, Xing, Eric, Trac,
  Hy, Ho, Shirley, and Schneider, Jeff.
\newblock Fast function to function regression.
\newblock In \emph{Artificial Intelligence and Statistics}, pp.\  717--725,
  2015.

\bibitem[Oliva et~al.(2013)Oliva, P{\'o}czos, and
  Schneider]{oliva2013distribution}
Oliva, Junier~B, P{\'o}czos, Barnab{\'a}s, and Schneider, Jeff~G.
\newblock Distribution to distribution regression.
\newblock In \emph{ICML (3)}, pp.\  1049--1057, 2013.

\bibitem[Oliva et~al.(2014)Oliva, Neiswanger, P{\'o}czos, Schneider, and
  Xing]{oliva2014fast}
Oliva, Junier~B, Neiswanger, Willie, P{\'o}czos, Barnab{\'a}s, Schneider,
  Jeff~G, and Xing, Eric~P.
\newblock Fast distribution to real regression.
\newblock In \emph{AISTATS}, pp.\  706--714, 2014.

\bibitem[Oort \& Rasmusson(1971)Oort and Rasmusson]{oort1971atmospheric}
Oort, Abraham~H and Rasmusson, Eugene~M.
\newblock \emph{Atmospheric circulation statistics}, volume~5.
\newblock US Government Printing Office, 1971.

\bibitem[Palmer(2000)]{palmer2000predicting}
Palmer, Tim~N.
\newblock Predicting uncertainty in forecasts of weather and climate.
\newblock \emph{Reports on progress in Physics}, 63\penalty0 (2):\penalty0 71,
  2000.

\bibitem[P{\'o}czos et~al.(2013)P{\'o}czos, Singh, Rinaldo, and
  Wasserman]{poczos2013distribution}
P{\'o}czos, Barnab{\'a}s, Singh, Aarti, Rinaldo, Alessandro, and Wasserman,
  Larry~A.
\newblock Distribution-free distribution regression.
\newblock In \emph{AISTATS}, pp.\  507--515, 2013.

\bibitem[Ravanbakhsh et~al.(2016)Ravanbakhsh, Oliva, Fromenteau, Price, Ho,
  Schneider, and P{\'o}czos]{ravanbakhsh2016estimating}
Ravanbakhsh, Siamak, Oliva, Junier~B, Fromenteau, Sebastian, Price, Layne, Ho,
  Shirley, Schneider, Jeff~G, and P{\'o}czos, Barnab{\'a}s.
\newblock Estimating cosmological parameters from the dark matter distribution.
\newblock In \emph{ICML}, pp.\  2407--2416, 2016.

\bibitem[Rematas et~al.(2013)Rematas, Fernando, Tommasi, and
  Tuytelaars]{rematas2013does}
Rematas, Konstantinos, Fernando, Basura, Tommasi, Tatiana, and Tuytelaars,
  Tinne.
\newblock Does evolution cause a domain shift?
\newblock In \emph{Proceedings VisDA 2013}, pp.\  1--3, 2013.

\bibitem[Risken(1996)]{risken1996fokker}
Risken, Hannes.
\newblock Fokker-planck equation.
\newblock In \emph{The Fokker-Planck Equation}, pp.\  63--95. Springer, 1996.

\end{thebibliography}
\bibliographystyle{bibliostyle}

\appendix
\section{Recurrent extension for DRN}\label{sect:appen_rdrn}
\begin{figure}[h!]
	\centering
	\includegraphics[width=0.4\columnwidth]{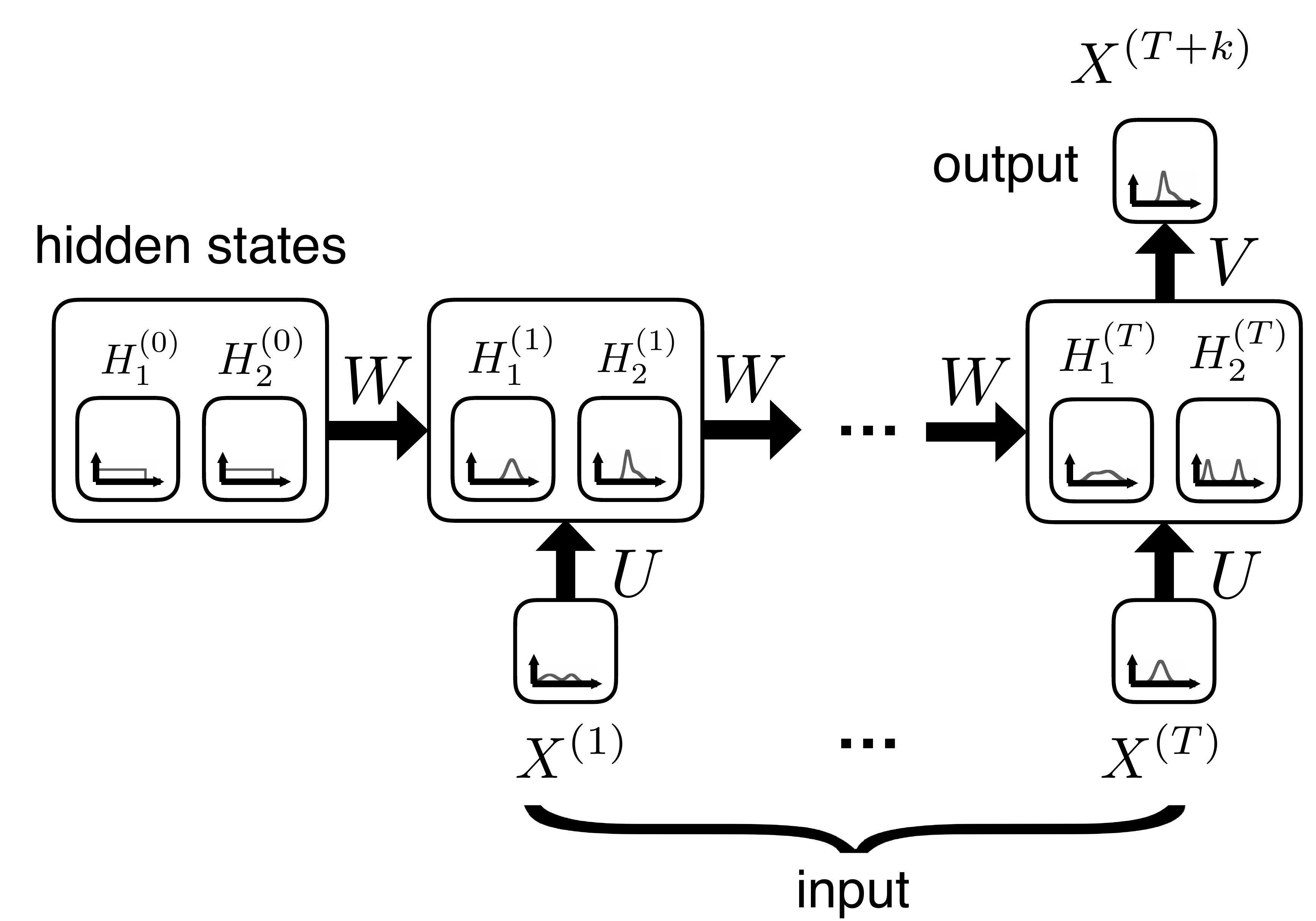}
	\caption{Recurrent distribution regression network}
	\label{fig:rdrn_multiple_dim}
\end{figure}

We introduce the recurrent distribution regression network (RDRN) which is a recurrent extension of
DRN. The input data is a distribution sequence as described in Section 3.1.
Figure~\ref{fig:rdrn_multiple_dim} shows an example network for RDRN, where the network takes in $T$
time steps of distributions to predict the distribution at $T+k$. The hidden state at each time step
may consist of multiple distributions. The arrows represent fully-connected weights. The
input-hidden weights $U$ and the hidden-hidden weights $W$ are shared across the time steps. $V$
represents the weights between the final hidden state and the output distribution. The bias
parameters for the hidden state nodes are also shared across the time steps. The hidden state
distributions at $t=0$ represents the `memory' of all past time steps before the first input and can
be initialized with any prior information. In our experiments, we initialize the $t=0$
hidden states as uniform distributions as we assume no prior information is known.

We formalize the propagation for the general case where there can be multiple distributions for each
time step in the data input layer and the hidden layer. Let $n$ and $m$ be the number of
distributions per time step in the data layer and hidden layers respectively. Propagation starts
from $t$=1 and performed through the time steps to obtain the hidden state distributions.
$X_i^{(t)}(r_i^{(t)})$ represents the input data distribution at node $i$ and time step $t$, when
the node variable is $r_i^{(t)}$. ${H}_k^{(t)}(s_k^{(t)})$ represents the density of the pdf of the
$k^{\text{th}}$ hidden node at time step $t$ when the node variable is $s_k^{(t)}$.
$\tilde{H}_k^{(t)}(s_k^{(t)})$  represents the unnormalized form. The hidden state distributions at
each time step is computed from the hidden state distributions from the previous time step and the
input data distribution from the current time step.
\begin{align}
\label{eq:rdrn_margin}
\tilde{H}_k^{(t)}\left(s_k^{(t)}\right) &=
\int_{{r_1}^{(t)},\cdots,{r_n}^{(t)},{s_1}^{(t-1)},\cdots,{s_m}^{(t-1)}}
\tilde{Q}\left(s_k^{(t)}|r_1^{(t)},\cdots, s_1^{(t-1)},\cdots\right) \\ \nonumber 
& \hspace{2cm} X_1^{(t)}\left(r_1^{(t)}\right)\cdots X_n^{(t)}\left(r_n^{(t)}\right)  
H_1^{(t-1)}\left(s_1^{(t-1)}\right)\cdots H_m^{(t-1)}\left(s_m^{(t-1)}\right)  \\ \nonumber
& \hspace{7cm}\,dr_1^{(t)} \cdots dr_n^{(t)}\,ds_1^{(t-1)} \cdots ds_m^{(t-1)} 
\end{align}
\begin{align}
\label{eq:rdrn_condprob}
\tilde{Q}(s_k^{(t)}|r_1^{(t)},\cdots, s_1^{(t-1)},\cdots) = e^{
	-E\left(s_k^{(t)}|r_1^{(t)},\cdots, s_1^{(t-1)},\cdots \right)}
\end{align}
The energy function is similar to the one in DRN and is given by
\begin{align}
\label{eq:rdrn_energy}
E\left(s_k^{(t)}|r_1^{(t)},\cdots, s_1^{(t-1)},\cdots \right) &= \sum_i^n u_{ki} \left(
\frac{s_k^{(t)}-r_i^{(t)}}{\Delta}\right)^2 
+\sum_j^m w_{kj} \left( \frac{s_k^{(t)}-s_j^{(t-1)}}{\Delta}\right)^2  \\ \nonumber
&+ b_{q,k}\left( \frac{s_k^{(t)}-\lambda_{q,k}}{\Delta} \right)^2  + b_{a,k}
\left|\frac{s_k^{(t)}-\lambda_{a,k}}{\Delta} \right|,
\end{align}
where for each time step, $u_{ki}$ is the weight connecting the $i^\text{th}$ input distribution to
the  $k^\text{th}$  hidden node. Similarly, for the hidden-hidden connections, $w_{kj}$ is the
weight connecting the $j^\text{th}$ hidden node in the previous time step to the $k^\text{th}$
hidden node in the current time step. As in DRN, the hidden node distributions are normalized before
propagating to the next time step. At the final time step, the output distribution is computed
from the hidden state distributions, through the weight vector $V$ and bias parameters at the output
node.

\section{Proofs for propositions for DRN theory}\label{sect:appen_DRNproofs}
Forward propagation in DRN can be written as a combination
of linear transformation and Hadamard products and then
normalized. For a node of order $n$, its activation $p_0$ is
\begin{equation}
\tilde{p}_0 = B_0 \circ (T_{w_1} \cdot p_1) \circ (T_{w_2} \cdot p_2)
\circ \cdots \circ (T_{w_n}\cdot p_n) = B_0 \circ \breve{\prod}^n_{i=1}
T_{w_i} \cdot p_i
\label{appen_eq:ptil}
\end{equation}
$\breve{\prod}$ is a symbol for Hadamard products and
\begin{equation}
p_0 = \tilde{p}_0 / | \tilde{p}_0|
\label{appen_eq:p}
\end{equation}
$\circ$ is the element wise Hadamard product operator, $\cdot$ is the matrix
multiplication operator. $B_0$ is a vector representing the bias term 
whose components are given by
\begin{equation}
(B_0)_i = \exp\left( 
-b_q \left(\frac{s_i - \lambda_q}{\Delta}\right)^2  
-b_a \left|\frac{s_i - \lambda_a}{\Delta}\right| \right)
\end{equation}
$T_{w_i}$ is a symmetric transformation matrix corresponding to the 
connections in DRN
whose elements are given by
\begin{equation}
(T_{w_i})_{qr} = \exp\left( - w_i 
\left(\frac{s_q - s_r}{\Delta}\right)^2 \right)
\end{equation}

\begin{proposition}
	A node connecting to a target node with zero weight $w=0$ has no effect on the
	activation of the target node.
\end{proposition}

\begin{proof}
	Without loss of generality, suppose $w_1=0$ in Eq. (\ref{appen_eq:ptil}), it is easy to show that
	\begin{equation}
	T_{w_1=0} \cdot p_1 = e = (1,1,\cdots 1)^t
	\end{equation}
	$e = (1,1,\cdots 1)^t$ is a vector with all ones. Using the identity,
	$u \circ e = u$ for any vector $u$, the term $T_{w_1} \cdot p_1$ drops
	out from Eq. (\ref{appen_eq:ptil}). Therefore $\tilde{p}_0$ does not depend
	on $p_1$.
\end{proof}
Similar to conventional neural networks, this is a mechanism for which DRN
can learn to ignore spurious nodes by setting their weights to zero or near zero.

\begin{proposition}
	For a node connecting to a target node with sufficiently large positive weight $w \to \infty$, the
	transformation matrix approaches the identity matrix: $T_w \to I$.
\end{proposition}

\begin{proof}
	Suppose $w \to \infty$ in Eq. (\ref{appen_eq:ptil}), it is easy to show that
	\begin{equation}
	T_{w_1 \to \infty} \cdot p_1 = I \cdot p_1 = p_1
	\end{equation}
\end{proof}
The consequence is that the identity mapping from one node to another can be realized.

\begin{lemma}
	Output of DRN is invariant to scaling the input by constant factors.
	\label{appen_prop:scale}
\end{lemma}

\begin{proof}
	Activation of one layer is invariant to scaling of the 
	activation in the previous layer. Using Eq. (\ref{appen_eq:ptil}), 
	performing
	the transformation $p_i \leftarrow c_i p_i$, where $c_i$ are
	scalars, leads to 
	$\tilde{p}_0 \leftarrow c_1 c_2 \cdots c_n \tilde{p}_0$,
	subsequent normalization makes $p_0$
	invariant to any scaling factors. The effects of scaling in a layer
	in the network is immediately eliminated in the next layer by
	normalization.
\end{proof}

\begin{proposition}
	Output of DRN is invariant to normalization of all hidden layers of DRN.
	\label{appen_prop:hidden_norm}
\end{proposition}

\begin{proof}
	For the purpose of proving, construct two networks identical in
	architecture and weights,
	forward propagate both networks, one network with
	nodes in the hidden layer normalized and the other network with nodes
	in the hidden layer unnormalized.
	
	Let $n_0$ be the number of input nodes and
	$n_1,\cdots $ be the number of hidden nodes in the hidden layers. Let activation in the input nodes be $p_i$, $i=1,\cdots n_0$. Let
	activation of the $i$th
	node in the $l$th
	hidden layer be $hn^{(l)}_i$ for the network with normalized
	hidden nodes and $hu^{(l)}_i$ for the network with unnormalized 
	hidden nodes.
	
	Performing forward 
	propagation for both networks,
	\begin{equation}
	\tilde{hn}^{(1)}_i = B_i^{(1)} \circ \breve{\prod}_{j=1}^{n_0} 
	T_{w^{(1)}_j} \cdot p_j
	\end{equation}
	\begin{equation}
	{hn}^{(1)}_i =  \frac{\tilde{hn}^{(1)}_i}{|\tilde{hn}^{(1)}_i|}
	\end{equation}
	\begin{equation}
	\tilde{hu}^{(1)}_i = B_i^{(1)} \circ \breve{\prod}_{j=1}^{n_0} 
	T_{w^{(1)}_j} \cdot p_j =  \tilde{hn}^{(1)}_i 
	\end{equation}
	For the second layer,
	\begin{eqnarray}
	\tilde{hn}^{(2)}_i & =  & 
	B_i^{(2)} \circ \breve{\prod}_{j=1}^{n_1}  T_{w^{(2)}_j} \cdot hn^{(1)}_j\\
	& = &
	\frac{B_i^{(2)} \circ \breve{\prod}_{j=1}^{n_1}  
		T_{w^{(2)}_j} \cdot \tilde{hn}^{(1)}_j }{
		\prod_{j'=1}^{n_1} |\tilde{hn}^{(1)}_{j'}| } \\ \nonumber
	&=& 
	\frac{B_i ^{(2)}\circ \breve{\prod}_{j=1}^{n_1}  
		T_{w^{(2)}_j} \cdot \tilde{hn}^{(1)}_j }{ z^{(1)} } 
	\end{eqnarray}
	$z^{(1)} =  \prod_{j'=1}^{n_1} |\tilde{hn}^{(1)}_{j'}|$ is the
	normalization scalar that factorizes out of the Hadamard product.
	\begin{equation}
	{hn}^{(2)}_i =  \frac{\tilde{hn}^{(2)}_i}{|\tilde{hn}^{(2)}_i|}
	\end{equation}
	\begin{eqnarray} \nonumber
	\tilde{hu}^{(2)}_i &=& B_i^{(2)} \circ \breve{\prod}_{j=1}^{n_1} 
	T_{w^{(2)}_j} \cdot \tilde{hu}^{(1)}_j \\ 
	\tilde{hu}^{(2)}_i &=&  \tilde{hn}^{(2)}_i z^{(1)}
	\label{appen_eq:huhn}
	\end{eqnarray}
	Using Eq. (\ref{appen_eq:huhn}), it can be shown that,
	\begin{equation}
	\tilde{hu}^{(l)}_i = \tilde{hn}^{(l)}_i \prod^{(l-1)}_{j=1} z^{(j)}
	= \tilde{hn}^{(l)}_i Z^{(l-1)}
	\label{appen_eq:huhn2}
	\end{equation}
	Without loss of generality, assume that the output consists of one node,
	let $pn$ and $pu$ be the final normalized output of the networks with normalized 
	hidden layers and unnormalized hidden layers respectively. Using Eq. (\ref{appen_eq:huhn2}),
	we shall proof that $pn = pu$.
	\begin{eqnarray}
	\tilde{pn} & = & (B^{(L)}) \circ \breve{\prod}_{j=1}^{n_{L-1}} (T_{w_j^{(L)}} \cdot hn^{(L-1)}_j) 
	\label{appen_eq:pn}
	\\ \nonumber
	\tilde{pu} & = & 
	(B^{(L)}) \circ \breve{\prod}_{j=1}^{n_{L-1}} (T_{w_j^{(L)}} \cdot \tilde{hu}^{(L-1)}_j) \\
	& = & (B^{(L)}) \circ \breve{\prod}_{j=1}^{n_{L-1}} (T_{w_j^{(L)}} \cdot \tilde{hn}^{(L-1)}_j) Z^{(L-1)}
	\label{appen_eq:pu}
	\end{eqnarray}
	
	Eq. (\ref{appen_eq:pn}) and (\ref{appen_eq:pu}) show that $\tilde{pn}$ and 
	$\tilde{pu}$ differ by a constant
	scalar. Therefore, after normalizing, $pn=pu$.
\end{proof}

Although normalization is theoretically unnecessary,
normalization step in Eq. (\ref{appen_eq:p}) provides numerical stability
and prevents numeric under-flows and over-flows.

\begin{definition}
	A node in DRN is said to be an order $n$ node when it is connected with non-zero weights from $n$
	incoming nodes in the previous layer.
\end{definition}

\begin{lemma}
	For an order $n$ node, components of $\tilde{p}_{0}$ 
	(which we denote as $\tilde{p}_{0i}$, where $i=1,\cdots, q$ and $q$ is the distribution discretization size),
	follow a power law of
	$n$th order cross terms of the components of
	connecting nodes.
\end{lemma}
\begin{proof} By rearranging the cross terms in the Hadamard product, we obtain the $n$th order cross terms from the components of the connecting nodes.
	\begin{eqnarray} \nonumber
	\tilde{p}_{0i} &=& {(B_0)}_i 
	\left( \sum_{j_1} (T_{w_1})_{i,j_1}(p_1)_{j_1} \right) \cdots
	\left( \sum_{j_n} (T_{w_n})_{i,j_n}(p_1)_{j_n} \right) \\ 
	&=& (B_0)_i 
	\sum_{j_1} \cdots \sum_{j_n}
	\underbrace{\left[ (T_{w_1})_{i,j_1} \cdots (T_{w_n})_{i,j_n} \right]}_\text{coefficients}
	\underbrace{\left[ (p_1)_{j_1} \cdots (p_n)_{j_n} \right]}_\text{cross terms} 
	\label{appen_eq:cross}
	\end{eqnarray}
\end{proof}
Writing in short hand notation, $\jid{1} = (j_1,\cdots j_n)$ where
the superscript indicates $J$ is the indices over the first layer.
Write
$\sum_{j_1} \cdots \sum_{j_n} = \sum_{\jid{1}}$ and consolidate the 
coefficients into a tensor, 
$\cid{i,\jid{1}} (w, B) = (B_0)_i (T_{w_1})_{i,j_1} \cdots (T_{w_n})_{i,j_n}$, and the cross terms into a tensor, 
$P_{\jid{1}} = (p_1)_{j_1} \cdots (p_n)_{j_n}$, where $w = (w_1,\cdots w_n)$.
Eq. (\ref{appen_eq:cross}) can be written compactly as,
\begin{equation}
\tilde{p}_{0i} = \sum_{\jid{1}} \cid{i,\jid{1}}(w,B) P_{\jid{1}}
\end{equation}
The normalization factor is,
\begin{equation}
| \tilde{p}_0 | = \sum_{\jid{1}} \sum_i \cid{i,\jid{1}} (w,B)  P_{\jid{1}} 
= \sum_{\jid{1}} \zid{\jid{1}} P_{\jid{1}}
\end{equation}
Finally, the normalized output, written in vector notation is,
\begin{equation}
p_{0} = \frac{\sum_{\jid{1}} \cid{\jid{1}} P_{\jid{1}}}
{\sum_{\jhd{1}} \zid{\jhd{1}} P_{\jhd{1}}}
\label{appen_eq:p2}
\end{equation}
where $p_{0} = (p_{01},p_{02},\cdots p_{0q})$ and 
$\cid{\jid{1}} = (\cid{1,\jid{1}},\cdots \cid{q,\jid{1}})$.

For a network with hidden layers, let $\hid{l}{1},\cdots \hid{l}{n_l}$
be the activation in the $l$th hidden layer. By Proposition 
\ref{appen_prop:hidden_norm}, we only consider unnormalized hidden layers.
The activation on the
first hidden layer can be written using Eq. (\ref{appen_eq:p2}) without
the normalization factor.
Using $\jidd{1}{\alpha}$ to denote the indices connecting to
$\hid{1}{\alpha}$, 
\begin{equation}
\hid{1}{\alpha} = \sum_{\jidd{1}{\alpha}} 
\cid{\jidd{1}{\alpha}}(w^{(1)}_\alpha, B^{(1)}_\alpha) P_{\jidd{1}{\alpha}}
\end{equation}
For the nodes in the second hidden layer, $\hid{2}{\alpha}$,
\begin{equation}
\hid{2}{\alpha} = \sum_{\jidd{2}{\alpha}} 
\cid{\jidd{2}{\alpha}}(w^{(2)}_\alpha, B^{(2)}_\alpha) H_{\jidd{2}{\alpha}}
\label{eq:h2}
\end{equation}
\begin{eqnarray} \nonumber
H_{\jidd{2}{\alpha}} &=& 
(\hid{1}{1})_{j_1} (\hid{1}{2})_{j_2} \cdots (\hid{1}{n_1})_{j_{n_1}} \\ \nonumber
& = & 
\left( \sum_{\jidd{1}{1}} 
\cid{j_1,\jidd{1}{1}} P_{\jidd{1}{1}}\right)
\cdots
\left( \sum_{\jidd{1}{n_1}} 
\cid{j_{n_1},\jidd{1}{n_1}} P_{\jidd{1}{n_1}}
\right)\\
& = & 
\prod_{\beta=1}^{n_1}
\left( \sum_{\jidd{1}{\beta}} 
\cid{j_{\beta},\jidd{1}{\beta}} P_{\jidd{1}{\beta}} \right)
\label{appen_eq:H}
\end{eqnarray}
Each of the $P_{\jidd{1}{\beta}}$ consists of cross terms of the input
distributions to order
$n_0$ ($n_0$ is the number of input nodes). $H_{\jidd{2}{\alpha}}$ is a product of
$n_1$ $P_{\jidd{1}{\beta}}$'s, hence
$H_{\jidd{2}{\alpha}}$ will be cross terms of the input distributions to order
$n_1 \times n_0$.

For a network of $L$ hidden layers with number of nodes, $n_1,n_2, \cdots n_L$,
the output consist of multiplications of the components of input distributions to
the power of $n_0 \times n_1 \cdots \times n_L$. In this way, DRN can fit high order functions
exponentially quickly by adding hidden layers.

\begin{proposition}
	For a node of order $n$, in the limit of small weights $|w_\alpha| \ll 1$
	for $\alpha=1,\cdots n$, the output activions, ${p}_{0}$
	can be approximated as a fraction of two
	linear combinations of the activations in the input nodes. 
	\label{appen_prop:linear}
\end{proposition}

\begin{proof}
	In the limit of small weights, one can expand $T_{w_\alpha}$ and keep the
	terms linear in $w_\alpha$. Then $T_{w_\alpha}$ is of the form
	\begin{eqnarray}
	T_{w_\alpha} & = & E + \mathcal{E}_{w_\alpha} + \mathcal{O}(w_\alpha^2)
	\end{eqnarray}
	where $E$ is a matrix with all ones in its elements,
	$\mathcal{O}(w_\alpha^2)$ is a matrix with elements of the order of
	$w_\alpha^2$ and $\mathcal{E}_{w_\alpha}$ is a matrix linear in $w_\alpha$ given by,
	\begin{equation}
	\mathcal{E}_{w_\alpha} = \left(
	\begin{array}{ccc}
	0            &-w_\alpha/\Delta^2 & -4 w_\alpha/ \Delta^2 \cdots \\
	-w_\alpha/\Delta^2 & 0            &  -w_\alpha/ \Delta^2 \cdots \\
	\vdots & &
	\end{array} \right)
	\end{equation}
	Using Eq.(\ref{appen_eq:ptil}) and dropping higher order terms,
	\begin{equation}
	\tilde{p}_0 \approx B_0 \circ ((E+\mathcal{E}(w_1)) \cdot p_1) 
	\circ ((E+\mathcal{E}(w_2)) \cdot p_2)
	\circ \cdots 
	\circ ((E+\mathcal{E}(w_n))\cdot p_n) \\ \nonumber
	\end{equation}
	Expanding and using the distributive property of Hadamard product
	\citep{million2007hadamard}, then dropping higher order terms
	in $w_i$ again,
	\begin{equation}
	\tilde{p}_{0i} \approx B_{0i} + 
	\sum_\alpha B_{0i} \circ (\mathcal{E}(w_\alpha) \cdot p_\alpha)_i 
	\end{equation}
	
	Upon normalization,
	\begin{equation}
	p_{0i} \approx \frac{B_{0i} + 
		\sum_\alpha B_{0i} \circ (\mathcal{E}(w_\alpha) \cdot p_\alpha)_i }
	{\sum_j  \left[ B_{0j} + \sum_{\alpha'} B_{0j} \circ (\mathcal{E}(w_{\alpha'}) \cdot p_{\alpha'})_j \right]}
	\end{equation}
\end{proof}

The consequence of proposition \ref{appen_prop:linear} is that by adjusting
the weights, DRN can approximate the unnormalized
output distribution to be a fraction of linear combinations of the input distribution. 

Indeed, the matrix $T_{w_\alpha}$ can be approximated by expanding
to $K$ orders in $w_\alpha$ 
with accuracy of expansion depending on  the magnitudes of $w_\alpha$. If
expansion is up to second order in $w_\alpha$ 
then the output is a fraction of quadratic expressions. If
the expansion in $w_\alpha$ is up to $K$ order then the resulting output is a fraction of polynomials of
$K$ order. At this point we wish to mention DRN's analogy to the well known Pad\'e approximant
\citep{baker1996pade}. Pad\'e approximant is a method of function approximation using fraction of
polynomials.

\section{Experimental details}\label{sect:appen_expt}
\begin{figure}
	\centering
	\includegraphics[width=0.5\columnwidth]{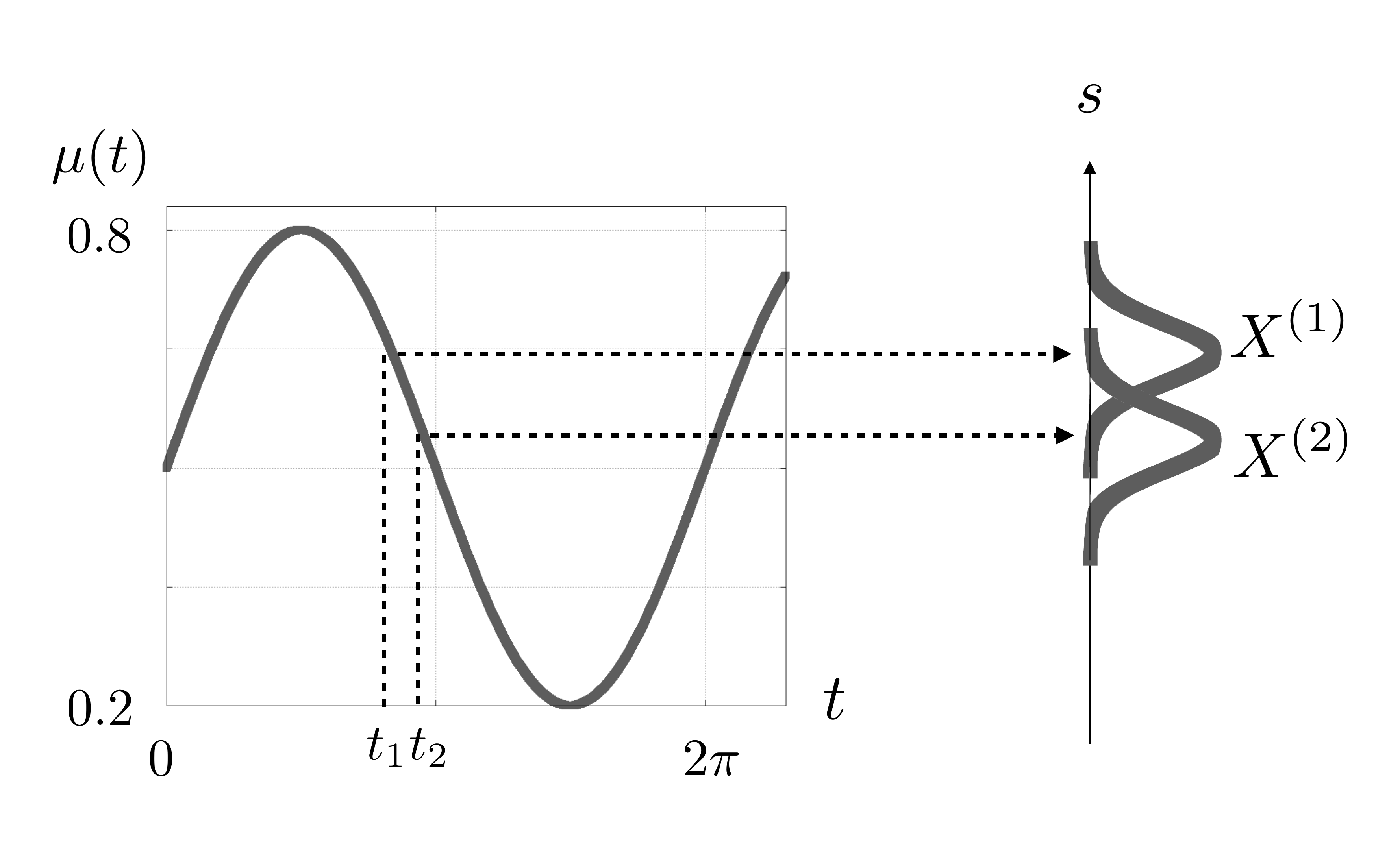}
	\caption{}
	\label{fig:sinegauss}
\end{figure}
\subsection{Shifting Gaussian}\label{sect:appen_shiftgauss}
We track a Gaussian distribution whose mean varies in the range $[0.2,0.8]$ sinusoidally over time
while the variance is kept constant at 0.1 (see Figure \ref{fig:sinegauss}). Given a few consecutive input distributions taken from
time steps spaced $\Delta t = 0.2$ apart, we predict the next time step distribution. For each data
we randomly sample the first time step from $[0, 2\pi]$. The distributions are truncated with
support of $[0,1]$ and discretized with $q$=100 bins. We found that for all methods, a history
length of 3 time steps is optimal. Following \citet{oliva2014fast} the regression performance is
measured by the $L_2$ loss, where lower $L_2$ loss is favorable. Table
\ref{table:appen_shiftgauss_model} shows the detailed network architectures used, for training size
of 20. $q = 100$ was used for the discretization of the distributions.

\begin{table}[h!]
	\centering
	\small
	Comparison of models tuned for best validation set result (Shifting Gaussian dataset, 20
	training data)
	\begin{tabular}{@{}cccccc@{}}  \hline
		& Test $L_2 (10^{-2})$  & $T$ &  Model description & $N_{p}$   & Cost function    \\  \hline
		DRN    & 4.90(0.46) & 3 & 3 - 2x10 - 1& 224   & JS divergence  \\
		RDRN   & \textbf{4.55(0.42)} &3&  T = 3, m = 5& 59  & JS divergence \\
		MLP    & 10.32(0.41) & 3 & 300 - 1x3 - 100& 1303  & MSE   \\
		RNN    & 17.50(0.89) &3  & T = 3, m = 10 & 2210   & MSE   \\
		3BE    & 22.30(1.89) & 3  & 30 basis functions, 20k RKS features & 6e+5  & $L_2$ loss   \\ \hline
	\end{tabular}
	\caption{Regression resutls for the shifting Gaussian dataset, with descriptions of the models.
		$L_2$ denotes the $L_2$ loss, $T$ is the optimal number of input time steps and $N_{p}$ is the
		number of model parameters used, MSE represents the mean squared error. A discretization of $q=100$
		is used for the distributions. For RDRN and RNN, $m$ is the number of nodes in the hidden state of
		each time step. For DRN and MLP (feedforward networks), the architecture is denoted as such: Eg. 3 -
		2x10 - 1: 3 input nodes, with 2 fully-connected hidden layers each with 10 nodes, and 1 output
		node.}
	\label{table:appen_shiftgauss_model}
\end{table}

\subsection{Climate Model}\label{sect:appen_climateOU}
With the long-term mean set at zero, the pdf has a mean of $\mu(t) = y\exp(-\theta t)$ and variance
of $\sigma^2(t) = D(1-e^{-2\theta t})/\theta$. $t$ represents time, $y$ is the initial point mass
position, and $D$ and $\theta$ are the diffusion and drift coefficients respectively. The diffusion
and drift coefficients are determined from real data measurements \citet{oort1971atmospheric}: $D =
0.0013,$  $\theta = 2.86$, and each unit of time corresponds to 55 days \citep{lin1987nonlinear}.
To create a distribution sequence, we first sample $y \in [0.02,0.09]$. For each $y$, we generate 6
Gaussian distributions at $t_0-4\delta$, $t_0-3\delta$,  ..., $t_0$ and $t_0+0.02$, with $\delta
= 0.001$ and $t_0$ sampled uniformly from $[0.01,0.05]$. The Gaussian distributions are truncated
with support of $[-0.01, 0.1]$. The regression task is as follows: Given the distributions at
$t_0-4\delta$, $t_0-3\delta$,  ..., $t_0$, predict the distribution at $t_0+0.02$.  With
different sampled values for $y$ and $t_0$, we created 100 training and 1000 test data. Table
\ref{table:appen_climateOU_model} shows the detailed network architectures used, for training size
of 100. $q = 100$ was used for the discretization of the distributions. 

\begin{table}[h!]
	\centering
	\small
	Comparison of models tuned for best validation set result (Climate model dataset, 100 training
	data)
	\begin{tabular}{@{}cccccc@{}}  \hline
		& Test $L_2 (10^{-2})$  & $T$ &  Model description & $N_{p}$  & Cost function    \\  \hline
		DRN     & 12.27(0.34) & 3  & 3 - 1x5 - 1& 44   & JS divergence \\
		RDRN   & \textbf{11.98(0.13)} & 5  & T = 5, m = 5 & 59 & JS divergence \\
		MLP     & 13.52(0.25) & 3   & 300 - 2x50 - 100 & 22700 & MSE  \\
		RNN     & 13.29(0.59) & 5 & T = 5, m = 50 & 12650  & MSE \\
		3BE     & 14.18(1.29) & 5  & 11 basis functions, 20k RKS features & 2.2e+5 & $L_2$ loss \\ \hline
	\end{tabular}
	\caption{Regression results for the climate model dataset, with descriptions of the models. $L_2$
		denotes the $L_2$ loss, $T$ is the optimal number of input time steps and $N_{p}$ is the number of
		model parameters used, MSE represents the mean squared error. A discretization of $q=100$
		is used for the distributions. For RDRN and RNN, $m$ is the number of nodes in the hidden state of
		each time step. For DRN and MLP (feedforward networks), the architecture is denoted as such: Eg. 3 -
		2x10 - 1: 3 input nodes, with 2 fully-connected hidden layers each with 10 nodes, and 1 output
		node.}
	\label{table:appen_climateOU_model}
\end{table}

\subsection{CarEvolution data}\label{sect:appen_cars}
The dataset consists of 1086 images of cars manufactured from the years 1972 to 2013. We split
the data into intervals of 5 years (i.e. 1970-1975, 1975-1980,~$\cdots$, 2010-2015) where each
interval has an average of 120 images. This gives 9 time intervals and for each interval, we create
the data distribution from the DeCAF (fc6) features \citep{donahue2014decaf} of the car images using
kernel density estimation. The DeCAF features have 4096 dimensions. Performing accurate density
estimation in very high dimensions is challenging due to the curse of dimensionality
\citep{gu2013nonparametric}. Here we make the approximation that the DeCAF features are independent,
such that the joint probability is a product of the individual dimension probabilities. The first 7 intervals were
used for the train set while the last 3 intervals were used for the test set, giving 5 training
and 1 test data. Table~\ref{table:appen_cars_model} shows the detailed network architectures used for the CarEvolution
dataset, for training size of 5. $q = 100$ was used for the discretization of the distributions.
Figure \ref{fig:carprobsamples} shows some samples of the distributions formed from the CarEvolution
dataset. The distributions' shapes are much more varied than simple Gaussian distributions.

\begin{table}[h!]
		\caption{Regression results for the CarEvolution dataset, with descriptions of the models. NLL
		denotes the negative log-likelihood, $T$ is the optimal number of input time steps and $N_{p}$ is
		the number of model parameters used, MSE represents the mean squared error. A discretization of
		$q=100$ is used for the distributions. For RDRN and RNN, $m$ is the number of nodes in the hidden
		state of each time step. For DRN and MLP (feedforward networks), the architecture is denoted as
		such: Eg. 3 - 2x10 - 1: 3 input nodes, with 2 fully-connected hidden layers each with 10
		nodes, and 1 output node.}
	\label{table:appen_cars_model}
	\centering
	\small
	Comparison of models tuned for best validation set result (CarEvolution dataset, 5 training data)
	\begin{tabular}{@{}cccccc@{}}  \hline
		& Test NLL  & $T$ &  Model description & $N_{p}$  & Cost function    \\  \hline
		DRN & 3.9663(2e-5)& 2 & (4096x2) - 1x1 - 4096 & 28676  & JS divergence   \\
		RDRN  & \textbf{3.9660(3e-4)}  & 2 & T = 2, m = 3 & 12313    & JS divergence \\
		MLP & 3.9702(6e-4) & 2& (4096x2x100) - 3x10 - (4096x100)& 1.2e+7   & MSE \\
		3BE & 3.9781(0.003) & 2 & 15 basis functions, 200 RKS features & 1.2e+7   & $L_2$ loss  \\
		EDD & 4.0405 & 1 & 8x8 kernel matrix for 8 training data & 64  & MSE \\  \hline
	\end{tabular}
\end{table}

\begin{figure}[h!]
	\centering
	\includegraphics[width=0.4\linewidth]{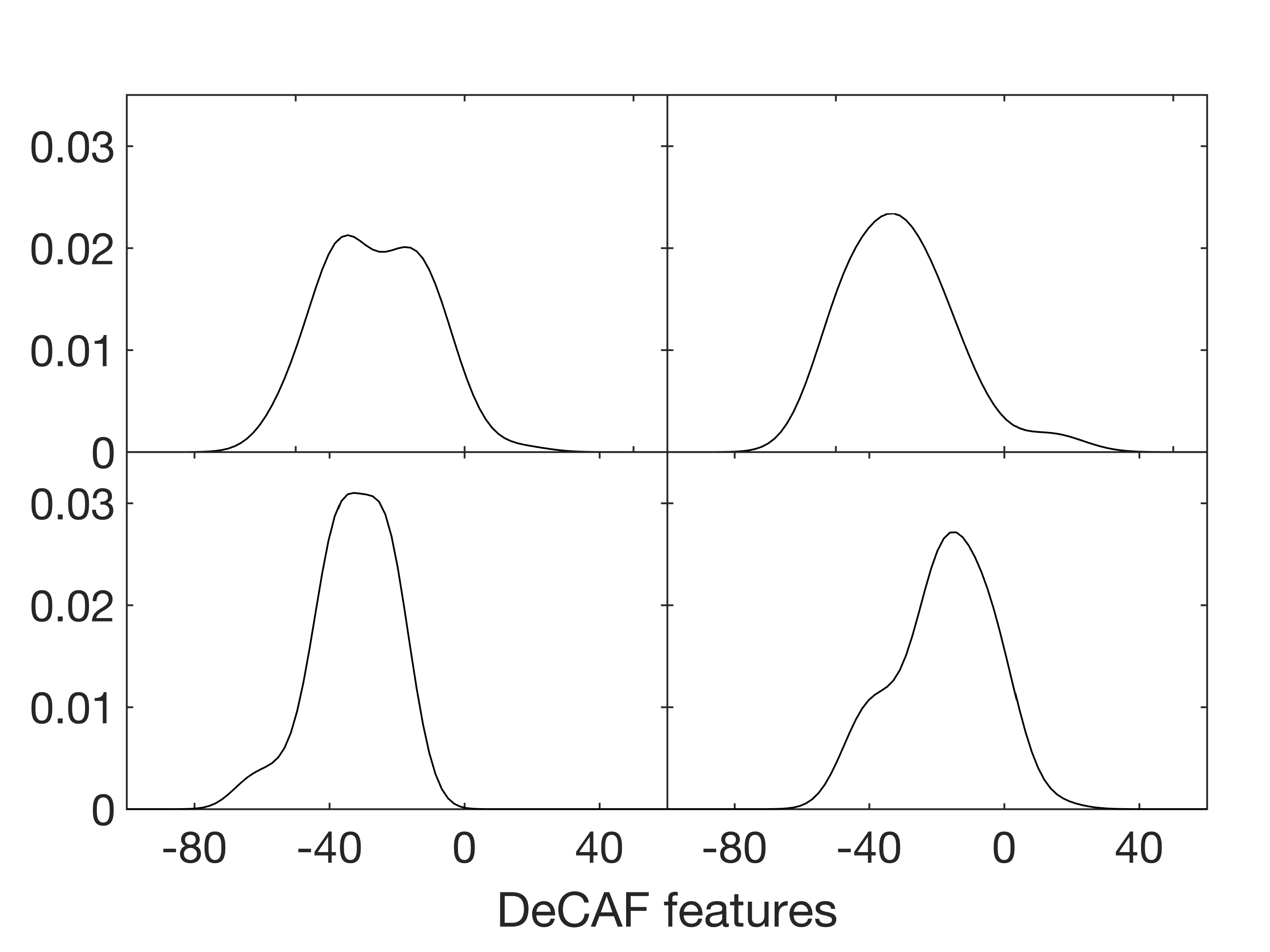}
	\caption{Samples of the DeCAF features \citep{donahue2014decaf} distributions for the CarEvolution
		dataset, showing varied distribution shapes.}
	\label{fig:carprobsamples}
\end{figure}

\subsection{Stock prediction}\label{sect:appen_stock}
We used 5 years of daily returns from 2011 to 2015 and performed exponential window averaging on the
price series \citep{murphy1999technical}. We also used a sliding-window training scheme
\citep{kaastra1996designing} to account for changing market conditions (details are in
\ref{sect:appen_stock}). The daily stock returns are the logarithmic returns. We used a sliding
window scheme where a new window is created and the model retained after every 300 days (the size
of the test set). For each window, the previous 500 and 100 days were used for training and
validation sets respectively. Table \ref{table:appen_stock_model} shows the detailed network
architectures used. $q = 100$ was used for the discretization of the distributions. The RDRN
architecture used is shown in Figure \ref{fig:stock_rdrn_nw}, where the data input consists past 3
days of distribution returns and one layer of hidden states with 3 nodes per time step is used.
Figure \ref{fig:stockprobsamples} shows some samples of the distributions formed from the stock
dataset. The distribution shapes are much more varied than simple Gaussian distributions. 

\begin{table}[h!]
			\caption{Regression results for the stock dataset, with descriptions of the models. NLL denotes the
		negative log-likelihood, $T$ is the optimal number of input time steps and $N_{p}$ is the number of
		model parameters used, MSE represents the mean squared error. A discretization of $q=100$
		is used for the distributions. For RDRN and RNN, $m$ is the number of nodes in the hidden state of
		each time step. For DRN and MLP (feedforward networks), the architecture is denoted as such: Eg. 3 -
		2x10 - 1: 3 input nodes, with 2 fully-connected hidden layers each with 10 nodes, and 1 output
		node.}
		\label{table:appen_stock_model}
	\centering
	\small
	Comparison of models tuned for best validation set result (Stock dataset, 200 training data)
	\begin{tabular}{@{}ccccccc@{}}  \hline
		& \multicolumn{2}{c}{Test NLL} & & &  \\
		&  1 day     & 10 days  & $T$ & Model description & $N_{p}$ & Cost function \\  \hline
		DRN  & \textbf{-473.93(0.02)} & -458.08(0.01) & 1 & No hidden layer & 9 & JS divergence      \\
		RDRN & -469.47(2.43)& \textbf{-459.14(0.01)} & 3 & T = 3, m = 3 & 37  & JS divergence  \\
		MLP & -471.00(0.04) & -457.08(0.98)& 3 &  (3x3x100) - 3x10 - 100& 10300 & MSE  \\
		RNN & -467.37(1.33) & -457.96(0.20)& 3 & T = 3, m = 10& 4210 & MSE  \\
		3BE & -464.22(0.16) & -379.43(11.8) & 1 &  14 basis functions, 1k RKS features & 14000  & $L_2$
		loss \\ \hline
	\end{tabular}
\end{table}

\begin{figure}
	\centering
	\includegraphics[width=0.6\linewidth]{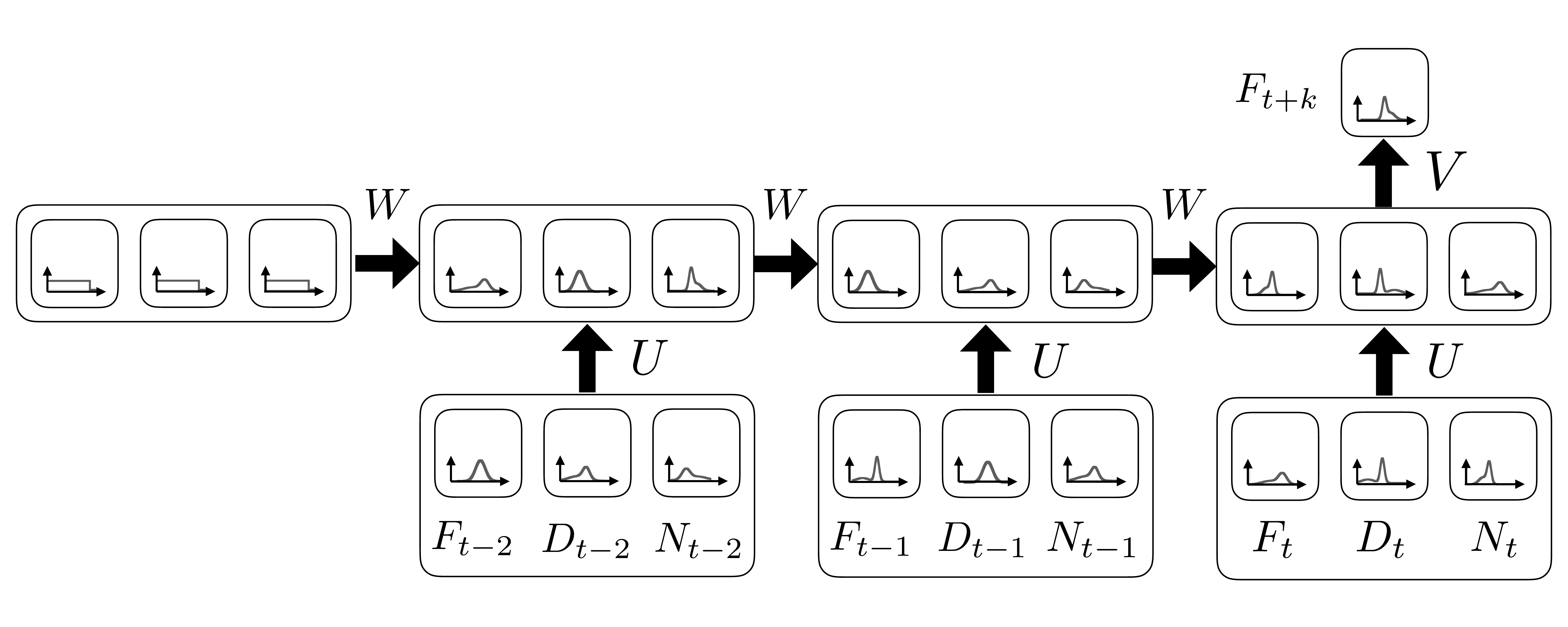}
	\caption{RDRN network for the stock dataset: past 3 days of distribution of returns of constituent
		companies in FTSE, DOW and Nikkei were used as inputs, to predict the future distribution of returns
		for constituent companies in FTSE. One layer of hidden states is used, with 3 nodes per hidden
		state.}
	\label{fig:stock_rdrn_nw}
\end{figure}

\begin{figure}[h!]
	\centering
	\begin{subfigure}[b]{0.6\columnwidth}
		\includegraphics[width=\columnwidth]{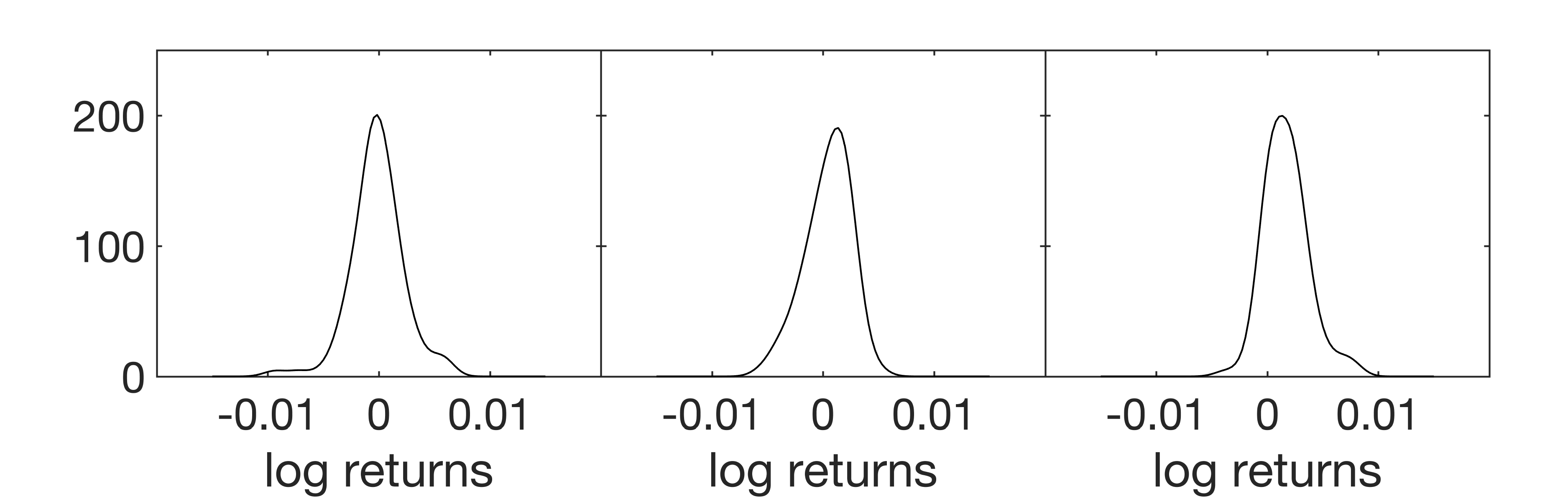}
		\caption{FTSE}
	\end{subfigure}
	\begin{subfigure}[b]{0.6\columnwidth}
		\includegraphics[width=\columnwidth]{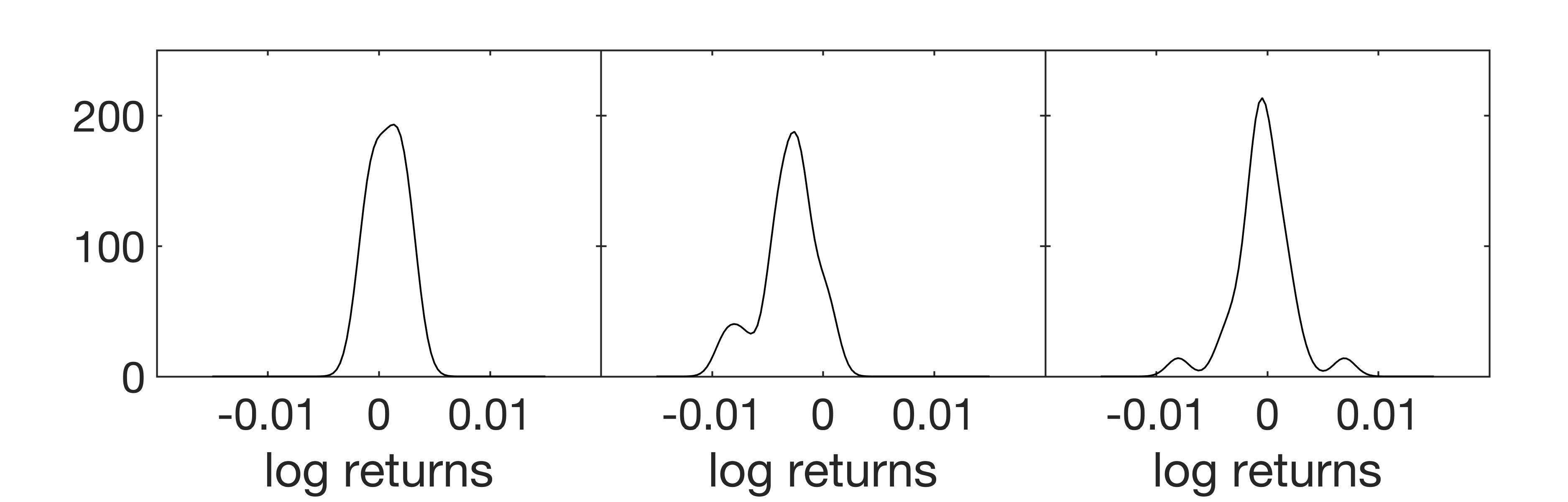}
		\caption{DOW}
	\end{subfigure}
	\begin{subfigure}[b]{0.6\columnwidth}
		\includegraphics[width=\columnwidth]{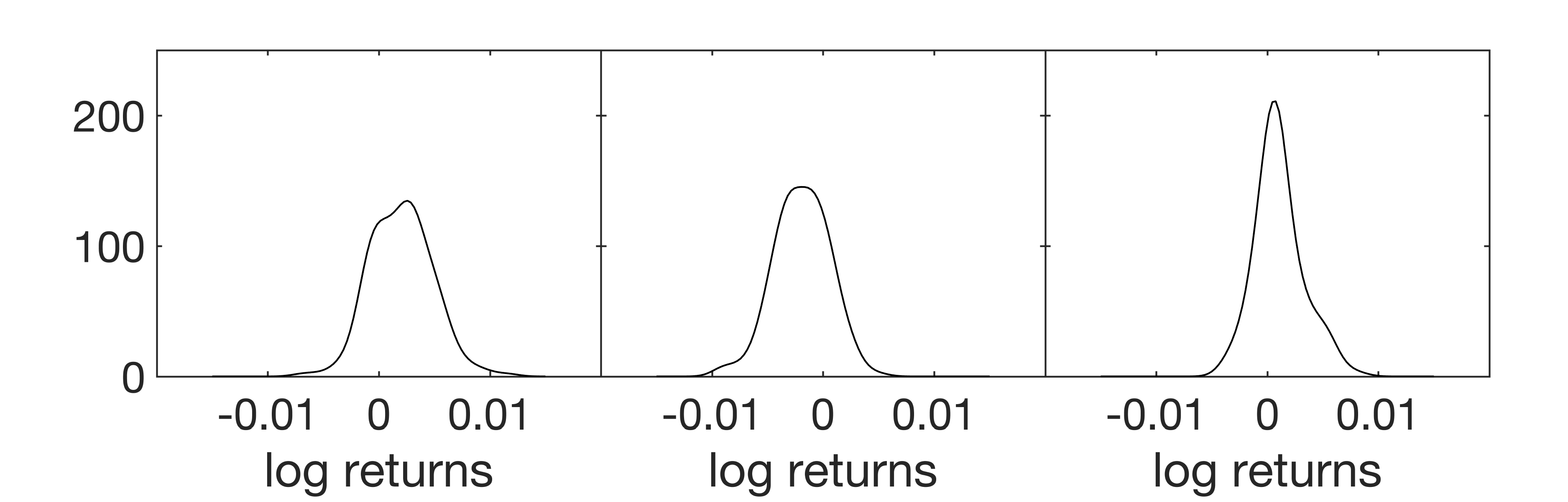}
		\caption{Nikkei}
	\end{subfigure}
	\caption{Samples of the data distributions formed from (a) FTSE, (b) DOW and (c) Nikkei constitutent
		companies' log returns. }
	\label{fig:stockprobsamples}
\end{figure}

\end{document}